\documentclass[a4paper]{article}[12pts]

\usepackage[english]{babel}
\usepackage[utf8x]{inputenc}
\usepackage[T1]{fontenc}

\normalsize

\usepackage[a4paper,top=1in,bottom=1in,left=1in,right=1in,marginparwidth=1in]{geometry}

\usepackage{setspace}
\usepackage{amsmath}
\usepackage{amssymb} 
\usepackage{amsthm}
\usepackage{amsfonts, mathtools}
\usepackage{algorithm,algpseudocode}
\usepackage{bm}
\usepackage{bbm}
\usepackage{comment}
\usepackage{graphicx}
\usepackage{multirow, booktabs}
\usepackage{caption}
\usepackage{subcaption}
\usepackage[colorinlistoftodos]{todonotes}
\usepackage{hyperref}
\hypersetup{colorlinks=true, allcolors=blue}

\newcommand{\E}{\mathbb{E}}

\DeclareMathOperator*{\argmax}{arg\,max}
\DeclareMathOperator*{\argmin}{arg\,min}

\newtheorem{proposition}{Proposition}
\newtheorem{lemma}{Lemma}

\newtheorem{assumption}{Assumption}
\newtheorem{theorem}{Theorem}

\usepackage{natbib}

\usepackage{algorithm,algpseudocode}

\title{Maximum Optimality Margin: A Unified Approach for Contextual Linear Programming and Inverse Linear Programming}
\author{Chunlin Sun$^*$, Shang Liu$^*$, Xiaocheng Li}
\date{\small 
Institute for Computational and Mathematical Engineering, Stanford University, chunlin@stanford.edu\\
Imperial College Business School, Imperial College London, (s.liu21, xiaocheng.li)@imperial.ac.uk}

\begin{document}
\maketitle

\onehalfspacing
\let\thefootnote\relax\footnotetext{$^*$ Equal contribution.}

\begin{abstract}
In this paper, we study the predict-then-optimize problem where the output of a machine learning prediction task is used as the input of some downstream optimization problem, say, the objective coefficient vector of a linear program. The problem is also known as predictive analytics or contextual linear programming. The existing approaches largely suffer from either (i) optimization intractability (a non-convex objective function)/statistical inefficiency (a suboptimal generalization bound) or (ii) requiring strong condition(s) such as no constraint or loss calibration. We develop a new approach to the problem called \textit{maximum optimality margin} which designs the machine learning loss function by the optimality condition of the downstream optimization. The max-margin formulation enjoys both computational efficiency and good theoretical properties for the learning procedure. More importantly, our new approach only needs the observations of the optimal solution in the training data rather than the objective function, which makes it a new and natural approach to the inverse linear programming problem under both contextual and context-free settings; we also analyze the proposed method under both offline and online settings, and demonstrate its performance using numerical experiments.  
\end{abstract}

\section{Introduction}

The predict-then-optimize problem considers a learning problem under a decision making context where the output of a machine learning model serves as the input of a downstream optimization problem (e.g. a
linear program). The ultimate goal of the learner is to prescribe a decision/solution for the downstream
optimization problem using directly the input (variables) of the machine learning model but without full
observation of the input of the optimization problem. A similar problem formulation was also studied as
prescriptive analytics \citep{bertsimas2020predictive} and contextual linear programming \citep{hu2022fast}.
While \cite{elmachtoub2022smart} justifies the importance of leveraging the optimization problem
structure when building the machine learning model, the existing efforts on exploiting the optimization
structure have been largely inadequate. In this paper, we delve deeper into the structural properties of
the optimization problem and propose a new approach called maximum optimality margin which builds a
max-margin learning model based on the optimality condition of the downstream optimization problem.
More importantly, our approach only needs the observations of the optimal solution in the training data
rather than the objective function, thus it draws an interesting connection to the inverse optimization
problem. The connection gives a new shared perspective on both the predict-then-optimize problem
and the inverse optimization problem, and our analysis reveals a scale inconsistency issue that arises
practically and theoretically for many existing methods.

Now we present the problem formulation and provide an overview of the existing techniques and related literature. Consider a linear program (LP) that takes the following standard form
\begin{align}
    \label{lp:std}
    \text{LP}(c,A,b) \coloneqq \min \ &  c^\top x,\\
    \text{s.t.\ } &  Ax=b, \ x\ge 0. \nonumber
\end{align}
where $c\in\mathbb{R}^n,$ $A\in\mathbb{R}^{m\times n},$ and $b\in\mathbb{R}^m$ are the inputs of the LP. In addition, there is an available feature vector $z\in\mathbb{R}^d$ that encodes the useful covariates (side information) associated with the LP. 

\textbf{Predict-then-optimize/Contextual LP}\

The problem of predict-the-optimize or contextual LP is stated as follows. A set of training data
$$\mathcal{D}_{\text{ML}}(T) \coloneqq \left\{\left(c_t, A_t, b_t, z_t\right)\right\}_{t=1}^T$$
consists of i.i.d. samples from an unknown distribution $\mathcal{P}$,
where $c_t\in\mathbb{R}^n,$ $A_t\in\mathbb{R}^{m\times n},$ $b_t\in\mathbb{R}^m$, and $z_t\in\mathbb{R}^d$. Throughout the paper, we assume the constraints have a fixed dimensionality for notational simplicity, while all the results can be easily extended to the case of variable dimensionalities. The goal of a conventional machine learning (ML) model is to identify a function $g(\cdot;\Theta): \mathbb{R}^d\rightarrow \mathbb{R}^n$ that best predicts the objective coefficient vector $c_t$ using the covariates $z_t$ with some model parametrized with $\Theta$. However, the predict-then-optimize problem has a slightly different pipeline for the testing phase. It aims to map from the observation of the context and the knowledge of the constraints to a decision (from data to decision):
$$(A_{\text{new}},b_{\text{new}},z_{\text{new}})\rightarrow x_{\text{new}}$$
without the observation of $c_{\text{new}}$, where the tuple $(c_{\text{new}}, A_{\text{new}},b_{\text{new}},z_{\text{new}})$ is a new test sample from $\mathcal{P}$. That is, for this new sample, the decision maker knows the constraints $(A_{\text{new}},b_{\text{new}})$, and the aim is to predict the unknown objective $c_{\text{new}}$ from $z_{\text{new}}$ and determine $x_{\text{new}}$ accordingly.

There are commonly three performance measures for the problem (as noted by \cite{chen2020online})
 
Prediction loss: $l_{\text{pre}}(\hat{\Theta}) = \E\left[\left\|c_{\text{new}}-\hat{c}_{\text{new}}\right\|_2^2\right]$

Estimate loss: $l_{\text{est}}(\hat{\Theta}) = \E\left[c_{\text{new}}^\top x_{\text{new}}-c_{\text{new}}^\top x_{\text{new}}^*\right]$

Suboptimality loss: $l_{\text{sub}}(\hat{\Theta}) = \E\left[\hat{c}_{\text{new}}^\top x_{\text{new}}^*-\hat{c}_{\text{new}}^\top \hat{x}_{\text{new}}\right]$
\smallskip

where $\hat{c}_{\text{new}}=g(z_\text{new};\hat{\Theta})$ is the predicted output of the ML model with parameters $\hat{\Theta}$. Here $\hat{x}_{\text{new}}$ and $x_{\text{new}}^*$ are the optimal solutions of LP$(\hat{c}_{\text{new}}, A_{\text{new}},b_{\text{new}})$ and LP$(c_{\text{new}}, A_{\text{new}},b_{\text{new}}),$ respectively. The expectations are taken with respect to this new sample. 

The prediction loss is aligned with the standard ML problems, where the L$_2$ loss can also be replaced by other proper metrics. The estimate loss captures the quality of the recommended decision $x_{\text{new}}$ under the true (unobserved) $c_{\text{new}}$ and it is a more natural loss given the interests in the downstream optimization task. The suboptimality loss measures how well the predicted $\hat{c}_{\text{new}}$ explains the realized optimal solution $x_{\text{new}}^*$. It is more commonly adopted in the inverse linear programming literature \citep{mohajerin2018data, barmann2018online, chen2020online}.

\textbf{Inverse linear programming}\

Our proposed approach to the predict-then-optimize problem also solves a seemingly unrelated problem -- inverse LP. In parallel to predict-then-optimize, the problem of inverse LP considers a set of training data
$$\mathcal{D}_{\text{inv}}(T) \coloneqq \left\{\left(x_t^*, A_t, b_t, z_t\right)\right\}_{t=1}^T.$$
Similarly to the previous case, the samples $(c_t, A_t, b_t, z_t)$'s are generated from an unknown distribution $\mathcal{P}$. Differently, for inverse LP, the optimal solution $x_t^*$ instead of the objective coefficient vectors $c_t$ is given in the training data. While the classic setting of inverse LP does not consider the context, i.e., $z_t\equiv 1$ for all $t,$ several recent works \citep{mohajerin2018data, besbes2021contextual} study the contextual case. The goal of the inverse LP is similar to that of the predict-then-optimize problem, to learn a function $g(\cdot;\Theta)$ that maps from the context $z_t$ to the objective $c_t$.

We emphasize that inverse LP is a much harder problem than contextual linear programming for several reasons. First, the observation of $x_t^*$ gives much less information than that of $c_t$, which makes the inverse problem information theoretically more challenging than the predict-then-optimize problem. Second, speaking of the objective function, directly minimizing the suboptimality loss $l_{\text{sub}}(\hat{\Theta})$ can lead to an unexpected failure. To see this, consider the naive prediction that always predicts $\hat{c}$ to be zero, always leading to an optimal zero suboptimality loss. Similar issues have also appeared in by ignored by the literature \citep{bertsimas2015data, mohajerin2018data, barmann2018online, chen2020online}, while our methods avoid such a problem (see Section \ref{sec:discussions}). 

In the following, we briefly review the representative existing approaches to tackle the two problems.

\textbf{Predicting the objective}\

The first class of approaches treats the predict-then-optimize problem as a two-step procedure. The learner first learns a function $g(\cdot;\hat{\Theta}): z \rightarrow c$ from the training data, and then prescribes the final output $x_\text{new}$ by the optimal solution of $\text{LP}(g(z_{\text{new}};\hat{\Theta}), A_{\text{new}},b_{\text{new}}).$ There are mainly two existing routes for how the parameter $\hat{\Theta}$ should be estimated from the training data. The first route is to ignore the constraint $(A_t, b_t)$ (in the training data) and treat the training problem as a pure ML problem \citep{ho2022risk}. The issue of this route is that there can be a misalignment between the ML loss $l_{\text{pre}}(\cdot)$ and the downstream optimization loss $l_{\text{est}}(\cdot)$ or $l_{\text{sub}}(\cdot)$. Empirically, it may cause sample inefficiency if one is interested in $l_{\text{est}}(\cdot)$. Theoretically, establishing a performance guarantee for the optimization loss, it requires a calibration condition \citep{bartlett2006convexity} which can be hard to satisfy/verify \citep{ho2022risk}. The second route is to employ the optimization loss $l_{\text{est}}(\cdot)$ for the estimation of $\hat{\Theta}$ \citep{elmachtoub2022smart, elmachtoub2020decision}. However, the loss function is generally non-convex in the parameter $\Theta$. A convex surrogate loss called SPO+ is proposed, but it also suffers from the misalignment issue when establishing a finite sample guarantee. To solve this problem, \cite{liu2021risk} show that the calibration condition holds, but an $O(T^{-1/2})$ convergence rate of the SPO+ loss only leads to an $O(T^{-1/4})$ convergence rate of the SPO loss, which is suboptimal (\cite{el2019generalization}). In the literature of inverse LP, the (convex) suboptimality loss $l_{\text{sub}}(\cdot)$ is often used instead of $l_{\text{est}}(\cdot)$. Specifically, our proposed approach falls in this category of first predicting the objective and then solving the LP. 

\textbf{Predicting the optimal solution}\

The second class of approaches treats the predict-then-optimize problem as a one-step procedure and aims to learn an end-to-end function that maps from the context $z_t$ directly to the optimal solution $x_t^*$ \citep{bertsimas2020predictive, hu2022fast}. This end-to-end treatment works well in the unconstrained setting. But for the constrained setting, the optimal solution of an LP generally stays at the corner of the feasible simplex, and the end-to-end mapping can hardly predict such a corner solution. In the domain of inverse LP, a recent work \citep{tan2020learning} also proposes a loss function minimizing the gap between the predicted optimal solution and the observed optimal solution. The idea is aligned with the early formulation of inverse optimization \citep{zhang1996calculating, ahuja2001inverse}. However, the formulation in \citep{tan2020learning} is more of a conceptual framework that is hardly computationally tractable.

Some existing studies on differentiable optimization also follow an end-to-end fashion and do not require objective functions as training data. Since the optimal solution of an LP is always at the corner, the change of optimal solution with respect to the coefficient is generally discontinuous, which restricts the possibility of backward propagation-based gradient descent learning. To address the discontinuity, one can either add some noise to the objective (\cite{berthet2020learning}) or add a regularization term in the objective (\citet{wilder2019melding}) to smooth the mapping from the objective to the optimal solution. Then, they can apply gradient descent to find the mapping. However, this smoothing technique only works for problems with fixed constraints but not the general random constraints as considered in our paper. The differentiable model training is based on a convex surrogate loss named Fenchel-Young loss \citep{blondel2020learning} which also enjoys some surrogate property \citep{blondel2019structured}, yet no finite sample bound has been obtained for the proposed algorithms. The training is also computationally costly: it requires a Monte-Carlo step (suppose we sample $M$ times) at each gradient computation, which needs solving $M$ times more linear programs than SPO+. An analogous idea of turning the non-differentiable case into the differentiable one is \citet{wilder2019melding} which adds a quadratic regularization term to the objective, while the consistency cannot be guaranteed.

\textbf{Consistency/Feasibility check.}\

For the inverse LP problem, a common approach in literature \citep{zhang1996calculating, ahuja2001inverse,boyd2004convex, bertsimas2015data, barmann2018online, besbes2021contextual} is to transform the knowledge of the optimal solution $x_t^*$ equivalently to constraints on the objective vector $c_t\in\mathcal{C}_t$ for some polyhedron $\mathcal{C}_t$ (through the optimality condition), and then utilize these $\mathcal{C}_t$ to make inference about the objective vector. There can be two issues with this approach. First, many existing results study the context-free case and require a non-empty intersection of the polyhedrons, i.e., $\cap_{t=1}^T \mathcal{C}_t \neq \emptyset.$ This requirement may fail in a noisy or contextual setting. Second, from a methodology perspective, it is difficult to relate the structure of $\mathcal{C}_t$ with the covariates $z_t.$ This $\mathcal{C}_t$ also highlights the difference between predict-then-optimize and inverse LP: the former observes $c_t$ in the training data, while the latter only knows $c_t\in\mathcal{C}_t$ for some polyhedron $\mathcal{C}_t$. 

\section{LP Optimality and Main Algorithm}
\label{sec:msa}

Now we first present some preliminaries on LP and then describe our main algorithm. Let $x^*=(x_1,...,x_n)^\top$ be the optimal solution of the LP$(c,A,b).$ The optimal basis $\mathcal{B}^*$  and its complement $\mathcal{N}^*$ of an LP are defined as follows
$$\mathcal{B}^* \coloneqq \{i:x_i^*>0\}, \ \ \mathcal{N}^* \coloneqq \{i:x_i^*=0\}.$$

For a set $\mathcal{B}\subset [n]$, we use $A_{\mathcal{B}}$ to denote the submatrix of $A$ with column indices corresponding to $\mathcal{B}$ and $c_{\mathcal{B}}$ to denote the subvector with corresponding dimensions. We make the following assumptions on the LP's nondegeneracy.

\begin{assumption}[Nondegeneracy]
All the LPs in this paper are \textit{nondegenerate}, i.e., satisfying the two conditions:
\begin{itemize}
\item[(a)] The LP is feasible and has a unique optimal solution. 
\item[(b)] All the submatrices $A_\mathcal{B}$ are invertible for $|\mathcal{B}|=m$. The optimal basis satisfies $|\mathcal{B}^*|=m$.
\end{itemize}
\label{assp_nondeg}
\end{assumption}

The assumption is standard in the literature of LP. It is a mild one in that any LP can satisfy the assumption under an arbitrarily small perturbation \citep{megiddo1989varepsilon}. We also note that our focus on the standard-form LP \eqref{lp:std} is without loss of generality because all the LPs can be written in the standard form and the results in this paper can be easily adapted to an LP of other forms.

The following lemma describes the optimality condition for an LP in terms of its input. 

\begin{lemma}[\cite{luenberger1984linear}]
\label{lem:opt_con}
Let $\mathcal{B}\subset [n]$ be a feasible basis of $\mathrm{LP}(c,A,b)$, i.e., $A^{-1}_{\mathcal{B}}b\ge 0,$ and let $\mathcal{N} = [n]\backslash \mathcal{B}$ be its complement. Then, under Assumption \ref{assp_nondeg}, the following inequality holds (element-wise)
\begin{equation}
   c_{\mathcal{N}}^\top - c_{\mathcal{B}}^\top A^{-1}_{\mathcal{B}} A_{\mathcal{N}} \ge 0 
   \label{lp_opt_con}
\end{equation}
if and only if $\mathcal{B}=\mathcal{B}^*.$
\end{lemma}

The result sets the foundation for the simplex method to solve LPs. When $\mathcal{B}=\mathcal{B}^*$, the vector ${p}^* \coloneqq c_{\mathcal{B}^*}^\top A^{-1}_{\mathcal{B}^*}\in\mathbb{R}^m$ will be the optimal solution of the dual LP of \eqref{lp:std}, also known as the dual price. 

\subsection{Maximum optimality margin}

Now we present the approach of maximum optimality margin. Consider the following training data
$$\mathcal{D}_{\text{inv}}(T) = \left\{\left(x_t^*, A_t, b_t, z_t\right)\right\}_{t=1}^T.$$
Note that the training data only consists of the observations of the optimal solution $x_t^*$, so all the algorithms and analyses apply to both the predict-then-optimize problem and the inverse LP problem.  Denote the sets of basic variables and non-basic variables of the $t$-th training sample as $\mathcal{B}_t^*\subseteq[n]$ and $\mathcal{N}_t^*\subseteq[n]$, respectively. Then the training data can be equivalently re-written as
$$\mathcal{D}_{\text{inv}}(T) = \left\{\left(x_t^*, A_t, b_t, z_t, \mathcal{B}^*_t, \mathcal{N}^*_t\right)\right\}_{t=1}^T.$$

Specifically, we start with a linear mapping from the covariates $z_t$ to the objective vector $c_t$, i.e.,
$$g(z_t;\Theta) \coloneqq \Theta z_t.$$
Our maximum optimality margin method solves the following optimization problem. 
\allowdisplaybreaks
\begin{alignat}{2}
    \label{prob:svm}
  \hat{\Theta}  \coloneqq \argmin_{\Theta\in\mathcal{K}} \ \ & \frac{\lambda}{2} \|\Theta\|_2^2+ \frac{1}{T}\sum_{t=1}^T \|s_t\|_1 \nonumber\\
    \text{s.t. }\ & \hat{c}_t = \Theta z_t, \ \ t=1,...,T,\\
    & \hat{c}_{t,\mathcal{N}^*_t}^\top - \hat{c}_{t,\mathcal{B}^*_t}^\top A_{t,\mathcal{B}^*_t}^{-1} A_{t,\mathcal{N}^*_t}&&\ge 1_{|\mathcal{N}^*_t|} - s_t , \ \ \nonumber\\
    &  &&t=1,...,T,\nonumber
\end{alignat}
where the decision variables are $\Theta \in \mathbb{R}^{n\times d}$, $\hat{c}_t\in\mathbb{R}^{n}$, and $s_t\in\mathbb{R}^{|\mathcal{N}_t|}$. $\mathcal{K}$ is a convex set to be determined. 

To interpret the optimization problem \eqref{prob:svm}: 

The equality constraints encode the linear prediction function, and $\hat{c}_t$ is the predicted value of $c_t$ under $\Theta$. 

For the inequality constraints, the vector $1_{|\mathcal{N}_t^*|}\in\mathbb{R}^{|\mathcal{N}_t^*|}$ is an all-one vector of dimension $|\mathcal{N}_t^*|$.
The left-hand-side of the inequality comes from the optimality condition \eqref{lp_opt_con} in Lemma \ref{lem:opt_con}, while the right-hand-side represents the \textit{slackness} or \textit{margin} of the optimality condition. From the objective function, it is easy to see that the optimal solution will always render $s_t\ge 0.$ Then, when $s_t \in [0,1]$ element-wise for all $t$, the inequality constraints fulfill the optimality condition \eqref{lp_opt_con} perfectly, which means that the optimal solution $\hat{\Theta}$ is consistent with all the training samples. In other words, when $s_t \in [0,1]$ element-wise for all $t$, if we predict the objective coefficient $c_t$ with $\hat{c}_t = \hat{\Theta} z_t$, the two LPs, LP$(c_t, A_t, b_t)$ and LP$(\hat{c}_t, A_t, b_t)$ share the same optimal solution for all $t$ (by Lemma \ref{lem:opt_con}), and it thus results in a zero estimate loss $l_{\text{est}}$ and a zero suboptimality loss $l_{\text{sub}}$ on the training data. When some coordinate of $s_t$ is greater than 1, it means the optimal solution under the predicted objective $\hat{c}_t = \hat{\Theta} z_t$ no longer satisfies the optimality condition of the original LP$(c_t,A_t,b_t)$, and consequently, this results in a mismatch between the two optimal solutions of LP$(c_t, A_t, b_t)$ and LP$(\hat{c}_t, A_t, b_t)$, and thus a non-zero loss of $l_{\text{est}}$ and $l_{\text{sub}}$.

For the objective function, the second part is justified by the above discussion on the inequality constraints. We desire to have a smaller value of $s_t$ as this leads to better satisfaction of the optimality condition. The first part of the objective function regularizes the parameter $\Theta$. The rationale is that the left-hand-side of the inequality constraints that represents the realized margin of the optimal condition, scales linearly with $\Theta$. We hope the margin to be large but do not want that the large margin is due to the scaling up of $\Theta.$

We note that the optimization problem \eqref{prob:svm} has linear constraints and a quadratic objective, and thus it is convex. Algorithm \ref{alg:MOM} describes the procedure of maximum optimality margin for solving the predict-then-optimize problem and the inverse LP problem. It is a two-step procedure in that it first estimates the parameter and then solves the optimization problem based on the estimate $\hat{\Theta}$.

\begin{algorithm}[ht!]
    \caption{Maximum Optimality Margin (MOM)}
    \label{alg:MOM}
    \begin{algorithmic}[1] 
    \Require Dataset $\mathcal{D}(T)=\{( {x}_t^*, {A}_t, {b}_t,z_t,\mathcal{B}_t^*,\mathcal{N}_t^*)\}_{t=1}^{T}$, test sample $(A_{\text{new}},b_{\text{new}}, z_{\text{new}})$, $\lambda$, $\mathcal{K}$
    \State Solve the optimization problem \eqref{prob:svm} and obtain the estimate $\hat{\Theta}$
    \State Predict $\hat{c}_{\text{new}}=\hat{\Theta} z_{\text{new}}$ and solve the following LP
    \begin{align*}
 \min \ &  \hat{c}_{\text{new}}^\top x,\\
    \text{s.t.\ } &  A_{\text{new}}x=b_{\text{new}}, \ x\ge 0. \nonumber
\end{align*}
    \State Denote its optimal solution as $\hat{x}_{\text{new}}$
    \Ensure $\hat{x}_{\text{new}}$ and $\hat{\Theta}$ 
    \end{algorithmic}
\end{algorithm}

\subsection{Interpreting the formulation}
\label{sec:interprete}

The key idea of the MOM formulation is that it does not explicitly minimize the error of predicting the objective $c_t$'s as a stand-alone ML problem, but it aims to minimize the violation -- equivalently, maximize the margin -- of the optimality conditions of the training samples (the inequality constraint in \eqref{prob:svm}). Intuitively, as long as we find a prediction $\hat{c}_t$ that shares the same optimal solution $x_t^*$ with $c_t$ for most $t$ (hopefully), we should not bother with the error between $\hat{c}_t$ and $c_t.$ This distinguishes from all the existing methods for predict-then-optimize and inverse LP in that it integrates the optimization problem's structure naturally into the ML training procedure. We make the following remarks. 

First, our method does not require the knowledge of $c_t$'s in the training data. This makes our method the first one that works for both predict-then-optimize and inverse LP problems. Beyond the point, this special feature of our method has a modeling advantage of \textit{scale consistency} or \textit{scale invariance}. Specifically, we note that for the LP \eqref{lp:std}, both the objective vectors $c$ and $\alpha c$ for any $\alpha>0$ lead to the same optimal solution. Hence the loss of training an ML model should ideally bear a scale invariance in terms of the objective vector. That is, a prediction of $\hat{c}$ and a prediction of $\alpha \hat{c}$ for any $\alpha>0$ should incur the same training loss as both of them lead to the same prescribed solution. Our method enjoys this scale invariance property, since it only utilizes the optimal solution $x_t^*$ in the training phase but does not involve $c_t$. In comparison, a method that minimizes the prediction loss $l_{\text{pre}}$ apparently does not have this scale invariance, so do some inverse LP algorithms (see Section \ref{sec_online_MOM}). This property can be critical for some application contexts such as revealed preference or stated preference \cite{beigman2006learning,zadimoghaddam2012efficiently} where one aims to predict the utility $c_t$ of a customer based on observed covariates $z_t$. In such contexts, even if we have observations of the $c_t$'s through surveying the customers' utilities, these observations might suffer from some scale contamination because each customer may have their own scale for measuring utilities. And therefore, striving for an accurate prediction of the observed $c_t$ can be misleading. 

Second, the inequality constraint in \eqref{prob:svm} is critical. Two tempting alternatives are to replace the inequality constraint of \eqref{prob:svm} with the following:
\allowdisplaybreaks
\begin{align}
\hat{c}_{t,\mathcal{N}^*_t}^\top - \hat{c}_{t,\mathcal{B}^*_t}^\top A_{t,\mathcal{B}^*_t}^{-1} A_{t,\mathcal{N}^*_t} & \ge 0, \label{alter1} \\
\hat{c}_{t,\mathcal{N}^*_t}^\top - \hat{c}_{t,\mathcal{B}^*_t}^\top A_{t,\mathcal{B}^*_t}^{-1} A_{t,\mathcal{N}^*_t} & \ge -s_t. \label{alter2}
\end{align}
For \eqref{alter1}, it directly employs the optimality condition \eqref{lp_opt_con} in Lemma \ref{lem:opt_con}. The issue with this formulation is that there may exist no $\Theta$ that is consistent with all the training data, and thus \eqref{alter1} will lead to an infeasible problem. In comparison, the inequality constraints of \eqref{prob:svm} can be viewed as a softer version of \eqref{alter1}. For \eqref{alter2}, it suffers a similar scale invariance problem as mentioned above. Specifically, the constraint \eqref{alter2} will drive $\hat{\Theta}\rightarrow 0$ as the left-hand-side of \eqref{alter2} scales linearly with $\hat{\Theta}$. To prevent this, one can impose an additional unit sphere constraint by $\|\Theta\|_2=1$ but this will lead to a non-convexity issue.

Lastly, we note a few additional points for the formulation. First, the formulation does not involve $b_t$ in the optimization problem \eqref{prob:svm} either. This is due to that $b_t$'s information is encoded by $(A_t, \mathcal{B}_t^*, \mathcal{N}_t^*)$ under the standard-form LP \eqref{lp:std}. Second, the formulation is presented under a linear mapping of $z_t$ to $c_t$; non-linearity dependency can be introduced by kernelizing the original covariates $z_t$ \citep{hofmann2008kernel}. Specifically, we can replace the original feature $z_{t_0}$ by a new $T$-dimensional feature with $(\kappa(z_{t_0}, z_t))_{t=1}^T$ for some kernel function $\kappa(\cdot, \cdot)$. Third, the MOM formulation aims to find a parameter $\hat{\Theta}$ that best satisfies the optimality condition, and it measures the quality of satisfaction by the total margin of optimality condition violation. It does not strive for a structured prediction of the optimal bases $\mathcal{B}_t$'s which will lead to a less tractable formulation such as a structured support vector machine (See Appendix \ref{sec_struct_svm}).

\section{Theoretical Analysis}
\label{sec:ta}

Now we analyze the maximum optimality margin method under both offline and online settings. We make the following boundedness assumptions mainly for notation simplicity. Specifically, we note that the parameter $\bar{\sigma}$ captures some ``stability'' of the LP's optimal solution under perturbation of the constraint matrix. While our algorithm utilizes the optimal solutions in the training data, better stability of the optimal solutions will lead to more reliable learning of the parameter.

\begin{assumption}[Boundedness]
\label{assp:bdd}
Let the tuple $(c,A,b,z,\mathcal{B}^*,\mathcal{N}^*)$ be a sample drawn from the distribution $\mathcal{P}$. The following assumptions hold with probability $1$.
\begin{itemize}
        \item[(a)]
            There exists a constant $\bar{\sigma}$ such that $ \sigma_{\max}(A_{\mathcal{B}^*}^{-1})\leq \bar{\sigma}$, where $\sigma_{\max}(\cdot)$ denotes the largest singular value function of a matrix.
        \item[(b)]
            The covariates vector $z$ is bounded by 1, i.e., $\|z\|_2\leq1$.
        \item[(c)]
            All entries of the LP's input  $(A,b,c)$ are within $[-1,1]$. 
        \item[(d)]
        The feasible region $\{x\geq0:Ax=b\}$ is bounded by the 2-norm unit ball.
    \end{itemize}
\end{assumption}

\subsection{Separable case}

We begin our discussion with the separable case where there exists a parameter that meets the optimality conditions on all the training samples with a margin.

\begin{assumption}
\label{assp:theta}
There exists $\Theta^*$ such that the following inequality holds element-wise almost surely,
\begin{equation}
   \hat{c}_{\mathcal{N}^*}^\top - \hat{c}_{\mathcal{B}^*}^\top A^{-1}_{\mathcal{B}^*} A_{\mathcal{N}^*} \ge 1
   \label{eqn:Theta_sepa}
\end{equation}
where $\hat{c}=\Theta^* z$ and $(c,A,b,z,\mathcal{B}^*,\mathcal{N}^*)$ is sampled from $\mathcal{P}.$ Suppose $\|\Theta^*\|_F\leq \bar{\Theta}$ for some constant $\bar{\Theta}>0$.
\end{assumption}

The assumption is weaker than assuming $c=\Theta^* z$ almost surely. It only requires the existence of a linear mapping $\hat{c}=\Theta^* z$ that meets the optimality condition almost surely. It can happen that Assumption \ref{assp:theta} holds but $c\neq \hat{c}$ almost surely. The right-hand-side of \eqref{eqn:Theta_sepa} changes from that of \eqref{lp_opt_con} in Lemma \ref{lem:opt_con} from 0 to 1. The change is not essential when the LP is nondegenerate almost surely; one can always scale up the parameter $\Theta^*$ to meet this margin requirement.

\begin{proposition}
    \label{prop:ofsprt}
Under Assumptions \ref{assp_nondeg}, \ref{assp:bdd} and \ref{assp:theta}, let $\hat{\Theta}$ denote the output of Algorithm \ref{alg:MOM} with $T$ training samples, $\lambda=\frac{1}{\sqrt{T}}$ and $\mathcal{K}=\{\Theta\in\mathbb{R}^{n\times d}:\|\Theta\|_F\leq\bar{\Theta}\}$. Suppose $(c_{\text{new}},A_{\text{new}},b_{\text{new}},z_{\text{new}})$ is a new sample from $\mathcal{P}$ and denote $\hat{c}_{\text{new}}=\hat{\Theta} z_{\text{new}}.$ Let $x^*_{\text{new}}$ and $\hat{x}_{\text{new}}$ denote the optimal solutions of LP$(c_{\text{new}}, A_{\text{new}},b_{\text{new}})$ and  LP$(\hat{c}_{\text{new}}, A_{\text{new}},b_{\text{new}})$.
Then we have
\begin{equation}
 \E\left[\mathbb{P}\left(x^*_{\text{new}}=\hat{x}_{\text{new}}\right)\right]\geq1-\frac{28+8\bar{\Theta}^2+7m\bar{\sigma}^2}{\sqrt{T}}, \label{bound1}
\end{equation}
where $m$ is the number of constraints in each LP, the expectation is taken with respect to the training samples, and the probability is with respect to the new test sample.
\end{proposition}

The proposition states that under the separable case, the algorithm can predict the optimal solution accurately with high probability. We note that the result does not concern the prediction error in terms of the objective vector, and this fact is aligned with the design of MOM that focuses on the optimality condition instead of the objective vector. Intuitively, even when one makes some error in predicting the objective $c_{\text{new}},$ the prediction of the optimal solution can still be accurate due to the simplex structure of LP. The proof of the proposition mimics the algorithm stability analysis  \citep{bousquet2002stability,shalev2010learnability} where the regularization term in \eqref{prob:svm} plays a role of \textit{stabilizing} the estimation. Here the performance bound is presented in the sense of on expectation. We remark that a high probability bound (removing the expectation in \eqref{bound1}) can also be achieved following the recent analysis of  \citep{feldman2019high}.

\subsection{Inseparable case}

Now we move on to the inseparable case where Assumption \ref{assp:theta} no longer holds. Let $D_{\text{new}} = (c,A,b,z, \mathcal{B}^*,\mathcal{N}^*)$ denote a new sample from $\mathcal{P}$. Define the margin-violation loss function
$$l(D_{\text{new}};\Theta) \coloneqq \|s\|_1 \text{ \ where \ } s \coloneqq 1_{|\mathcal{N}^*|} - \hat{c}_{\mathcal{N}^*}^\top - \hat{c}_{\mathcal{B}^*}^\top A_{\mathcal{B}^*_t}^{-1} A_{\mathcal{N}^*}$$
and $\hat{c}= \Theta z$. The loss function is inherited from the objective function of the MOM formulation \eqref{prob:svm}. For the inseparable case, we can derive the following generalization bound as Proposition \ref{prop:ofsprt}. 

\begin{proposition}
\label{prop:ofgen}
Under Assumptions \ref{assp_nondeg} and \ref{assp:bdd}, let $\hat{\Theta}$ denote the output of Algorithm \ref{alg:MOM} with $T$ training samples, under the choice of $\lambda=\frac{1}{\sqrt{T}}$ and $\mathcal{K}=\{\Theta\in\mathbb{R}^{n\times d}:\|\Theta\|_F\leq \bar{\Theta}\}$ for some $\bar{\Theta}>0$. Then we have
    \begin{align*}
        \mathbb{E}[l(D_{\text{new}};\hat{\Theta})]
        \leq
        \min_{ {\Theta}\in\mathcal{K}}\mathbb{E}[l( D_{\text{new}};{\Theta})]
        +\frac{28+8\bar{\Theta}^2+7m\bar{\sigma}^2}{\sqrt{T+1}},        
    \end{align*}
    where the expectation is taken with respect to both the new sample $D_{\text{new}}$ and the training samples.
\end{proposition}

The proposition follows the same analysis as Proposition \ref{prop:ofsprt} and the right-hand-side involves an additional term which will become zero under the separability condition of Assumption \ref{assp:theta}. In the previous section, we motivate the optimization formulation \eqref{prob:svm} of MOM in its own right. The following corollary reveals an interesting connection between its objective function and the suboptimality loss. Specifically, the margin violation loss optimized in MOM can be viewed as an upper bound of the suboptimality loss.

\begin{proposition}
The following inequality holds for any $\Theta\in\mathbb{R}^{n\times d}$,
$$l_{\text{sub}}(\Theta) \le \mathbb{E}[l( D_{\text{new}};{\Theta})].$$
Therefore, under Assumptions \ref{assp_nondeg} and \ref{assp:bdd}, and the same setup as Proposition \ref{prop:ofgen}, 
$$ l_{\text{sub}}(\hat{\Theta})
        \leq
        \min_{ {\Theta}\in\mathcal{K}}\mathbb{E}[l(D_{\text{new}};
        {\Theta})]
        +\frac{28+8\bar{\Theta}^2+7m\bar{\sigma}^2}{\sqrt{T}}.$$
    \label{prop:insep}
\end{proposition}
\vspace{-0.7cm}

The additional term on the right-hand-side in both Proposition \ref{prop:ofgen} and Proposition \ref{prop:insep} can be viewed as a measure of separability in terms of the LP's optimality condition. Specifically, if one can predict the objective $c$ very well with the covariates $z$, then this term will become close to zero; otherwise, this term will become large. While the standard measurement of the predictive power of $z$ concerns the minimum error in predicting $c$, it does not account for the structure of the downstream LP optimization problem.


\subsection{Online maximum optimality margin}

\label{sec_online_MOM}

The MOM formulation also enables an online algorithm which reduces the quadratic program of \eqref{prob:svm} to a sub-gradient descent step per sample. Specifically, we can view the margin violation of the $t$-th sample as a piecewise-linear function of $\Theta$, i.e., 
$$l(D_{t};\Theta) = \|s_t\|_1 = \left\|1_{|\mathcal{N}^*_t|} - \hat{c}_{t,\mathcal{N}^*_t}^\top - \hat{c}_{t,\mathcal{B}^*_t}^\top A_{t,\mathcal{B}^*_t}^{-1} A_{t,\mathcal{N}^*_t}\right\|_1$$
where $D_t$ represents the $t$-th sample in the dataset $\mathcal{D}(T).$
The function $l(D_{t};\Theta)$ is a convex function with respect to $\Theta$ and the piecewise linearity enables an efficient calculation of the sub-gradients. Algorithm \ref{alg:sgd} arises from an online sub-gradient descent algorithm with respect to the cumulative loss 
$\sum_{t=1}^T l(D_{t};\Theta).$

\begin{algorithm}[ht!]
    \caption{Online MOM Algorithm}
    \label{alg:sgd}
    \begin{algorithmic}[1] 
    \Require Dataset $\mathcal{D}(T)=\{( {x}_t^*, {A}_t, {b}_t,z_t,\mathcal{B}_t^*,\mathcal{N}_t^*)\}_{t=1}^{T}$, step size $\eta$, $\mathcal{K}$
    \State Initialize ${ {\Theta}}_1 =  {0}\in\mathbb{R}^{n\times d}$
    \For {$t=1,...,T$}
        \State Predict the objective by $\hat{ {c}}_t= {\Theta}_{t} {z}_t$
        \State  Solve the following LP and denote its optimal solution by $x_t$ \begin{align*}
 \min \ &  \hat{c}_{t}^\top x,\\
    \text{s.t.\ } &  A_{t}x=b_{t}, \ x\ge 0. \nonumber
\end{align*}
        \State Update 
        \begin{align*}
            { {\Theta}}_{t+1} = \text{Proj}_{\mathcal{K}}\left({ {\Theta}}_{t}-
            \eta \cdot\partial_{ {\Theta}} l(D_{t};\Theta_t)\right)
        \end{align*}
    \EndFor
    \Ensure $\{{x}_t\}_{t=1}^{T}, \{\Theta_{T}\}_{t=1}^{T+1}$ 
    \end{algorithmic}
\end{algorithm}

\begin{proposition}
    \label{prop:sgd}
    Under Assumptions \ref{assp_nondeg} and \ref{assp:bdd}, with the choice of $\eta=\frac{2\bar{\Theta}}{(\sqrt{n}+\bar{\sigma}\cdot mn)\sqrt{T}}$ and $\mathcal{K}=\{\Theta\in\mathbb{R}^{n\times d}:\|\Theta\|_F\leq \bar{\Theta}\}$ for some $\bar{\Theta}>0$,
    the outputs of Algorithm \ref{alg:sgd} satisfy
    \begin{align}
        \frac{1}{T}\mathbb{E}\left[\sum\limits_{t=1}^{T}\hat{c}_t^{\top}(x^*_t-x_t)\right]
        \leq
        &\mathbb{E}\left[\min_{ \Theta\in\mathcal{K}} \sum\limits_{t=1}^{T}l(D_t; \Theta)\right]
\nonumber \\&+\frac{3\bar{\Theta}\sqrt{n}+3\bar{\sigma}\bar{\Theta}\cdot mn}{\sqrt{T}}
        \label{bound2}
    \end{align}
where $D_{\text{new}} = (c_{\text{new}}, A_{\text{new}},b_{\text{new}}, z_{\text{new}})$ denotes a new sample.  

Moreover, under Assumption \ref{assp:theta}, let $\hat{c}_{\text{new}}=\hat{\Theta} z_{\text{new}}$ where $\hat{\Theta}$ is uniformly randomly sampled from $\{\Theta_t\}_{t=1}^{T+1}$. Let $x^*_{\text{new}}$ and $\hat{x}_{\text{new}}$ denote the optimal solutions of LP$(c_{\text{new}}, A_{\text{new}},b_{\text{new}})$ and  LP$(\hat{c}_{\text{new}}, A_{\text{new}},b_{\text{new}})$.
Then we have
\begin{equation}
 \E\left[\mathbb{P}\left(x^*_{\text{new}}=\hat{x}_{\text{new}}\right)\right]\geq1-\frac{3\bar{\Theta}\sqrt{n}+3\bar{\sigma}\bar{\Theta}\cdot mn}{\sqrt{T+1}} \label{bound3}
 \end{equation}
 where the expectation is taken with respect to both training data samples and the new sample.
\end{proposition}

The proposition states that the results in Proposition \ref{prop:ofsprt} and Proposition \ref{prop:ofgen} can also be achieved by a subgradient descent algorithm from an online perspective. The bound \eqref{bound2} also recovers the regret bounds in the online inverse optimization literature \citep{barmann2018online, chen2020online} under a contextual setting. However, the algorithms therein require the convex set $\mathcal{K}$ not containing the origin (which will significantly restrict the predictive power of $z_t\rightarrow c_t$). Otherwise, we find numerically that these existing algorithms will drive the parameter $\Theta_{t}\rightarrow 0$ and $\hat{c}_t \rightarrow 0$, and consequently provide no meaningful prediction of the optimal solution. This reinforces the point of scale invariance raised earlier. As mentioned earlier, the margin-based formulation of MOM resolves the issue of scale invariance, and as a result, Algorithm \ref{alg:sgd} can provide an additional bound \eqref{bound3} and also perform stably well in numerical experiments, which we refer to Appendix \ref{sec_online_experi}.

\subsection{Tighter bound under separability -- optimality-driven perceptron}

We conclude this section with Algorithm \ref{alg:prcpt} that achieves a faster rate under the separability condition (Assumption \ref{assp:theta}). Specifically, the algorithm is an online algorithm that utilizes the knowledge of separability. The idea is to treat each inequality constraint in \eqref{prob:svm} as an individual binary classification task of distinguishing non-basic variables from basic variables, and view the margin of optimality condition as the margin for classification. Then when the margin (the reduced cost in Algorithm \ref{alg:prcpt}) drops below a threshold of $1/2$, we perform an update of the parameter. 

\begin{algorithm}[ht!]
    \caption{Optimality-driven Perceptron Algorithm}
    \label{alg:prcpt}
    \begin{algorithmic}[1]
    \Require Dataset $\mathcal{D}(T)=\{( {x}_t^*, {A}_t, {b}_t,z_t,\mathcal{B}_t^*,\mathcal{N}_t^*)\}_{t=1}^{T}$
    \State Let ${ {\Theta}}_1 =  {0}\in\mathbb{R}^{n\times d}$
    \State \%\% We use $(M)_{i}$ to denote the $i$-th row of the matrix $M$
    \For {$t=1,...,T$}
\State  Predict the objective by $\hat{{c}}_t= {\Theta}_{t} {z}_t$
            \State  Solve the following LP and denote its optimal solution by $x_t$ \begin{align*}
 \min \ &  \hat{c}_{t}^\top x,\\
    \text{s.t.\ } &  A_{t}x=b_{t}, \ x\ge 0. \nonumber
\end{align*}
    \State Let $\Theta_{\text{tmp}} = \Theta_t$
    \For{$i=1,...,n$}
    \If{$i\notin\mathcal{N}_t^*$}
    \State Continue
    \EndIf
    \State  Predict the objective by $\hat{ {c}}_t^{(i)}= {\Theta}_{\text{tmp}} {z}_t$
        \State Compute the \textit{reduced cost} 
        $$
            r_t^{(i)\top} = \hat{c}_t^{(i)\top} -\hat{{c}}_{t,{\mathcal{B}}_t^*}^{(i)\top} A_{t,{\mathcal{B}}_t^*}^{-1}A_t
        $$
            \If{$r^{(i)}_{t,i}\le \frac{1}{2}$}
            \State Update \begin{align*}
                (\Theta_{\text{tmp}})_i
                &\leftarrow ({{\Theta}}_{\text{tmp}})_i+z_t^{\top}\\
            ({{\Theta}}_{\text{tmp}})_{\mathcal{B}_t^*}
               & \leftarrow
                ({{\Theta}}_{\text{tmp}})_{\mathcal{B}_t^*}
                -
                A_{\mathcal{B}_t^*}^{-1}A_t z_t
            \end{align*}
            \EndIf
        \EndFor
        \State Let ${{\Theta}}_{t+1}={{\Theta}}_{\text{tmp}}$
    \EndFor
    \Ensure $\{x_t\}_{t=1}^{T}$, $\{\Theta_{t}\}_{t=1}^{T+1}$ 
    \end{algorithmic}
\end{algorithm}

\begin{proposition}
    \label{prop:pcpt}
    Under Assumptions \ref{assp_nondeg}, \ref{assp:bdd} and \ref{assp:theta}, the number of mistakes made by Algorithm \ref{alg:prcpt} is independent of $T$,
    $$\#\left\{t\in[T]:x_t^*\neq x_t\right\}\leq \bar{\Theta}^2+\bar{\sigma}^2\bar{\Theta}^2m^2n.$$
In addition, suppose $(c_{\text{new}},A_{\text{new}},b_{\text{new}},z_{\text{new}})$ is a new sample from $\mathcal{P}$ and denote $\hat{c}_{\text{new}}=\hat{\Theta} z_{\text{new}}$ where $\hat{\Theta}$ is uniformly randomly sampled from $\{\Theta_t\}_{t=1}^{T+1}$. Let $x^*_{\text{new}}$ and $\hat{x}_{\text{new}}$ denote the optimal solutions of LP$(c_{\text{new}}, A_{\text{new}},b_{\text{new}})$ and  LP$(\hat{c}_{\text{new}}, A_{\text{new}},b_{\text{new}})$,
    $$
        \mathbb{E}\left[
            \mathbb{P}(x^*_{\text{new}}=\hat{x}_{\text{new}})
        \right]
        \geq
        1- \frac{\bar{\Theta}^2+\bar{\sigma}^2\bar{\Theta}^2m^2n}{T+1},
    $$
    where the expectation is taken with respect to both the training data and the new data.
\end{proposition}

Proposition \ref{prop:pcpt} provides an upper bound on the number of mistakes made by Algorithm \ref{alg:prcpt}. The algorithm improves the dependency on $T$ compared with the previous bounds in Proposition \ref{prop:ofsprt} and Proposition \ref{prop:sgd}. However, we note this improvement  sacrifices the dependency on either the problem size $m$ and $n$ or the constants of $\bar{\Theta}$ and $\bar{\sigma}.$ In this light, the result is more of theoretical interest that completes our discussion. In the numerical experiments in Appendix \ref{sec_online_experi}, we observe that the performance of Algorithm \ref{alg:prcpt} is indeed worse than Algorithm \ref{alg:MOM} and Algorithm \ref{alg:sgd}.

\begin{figure*}[ht!]
\centering
\includegraphics[width=\textwidth]{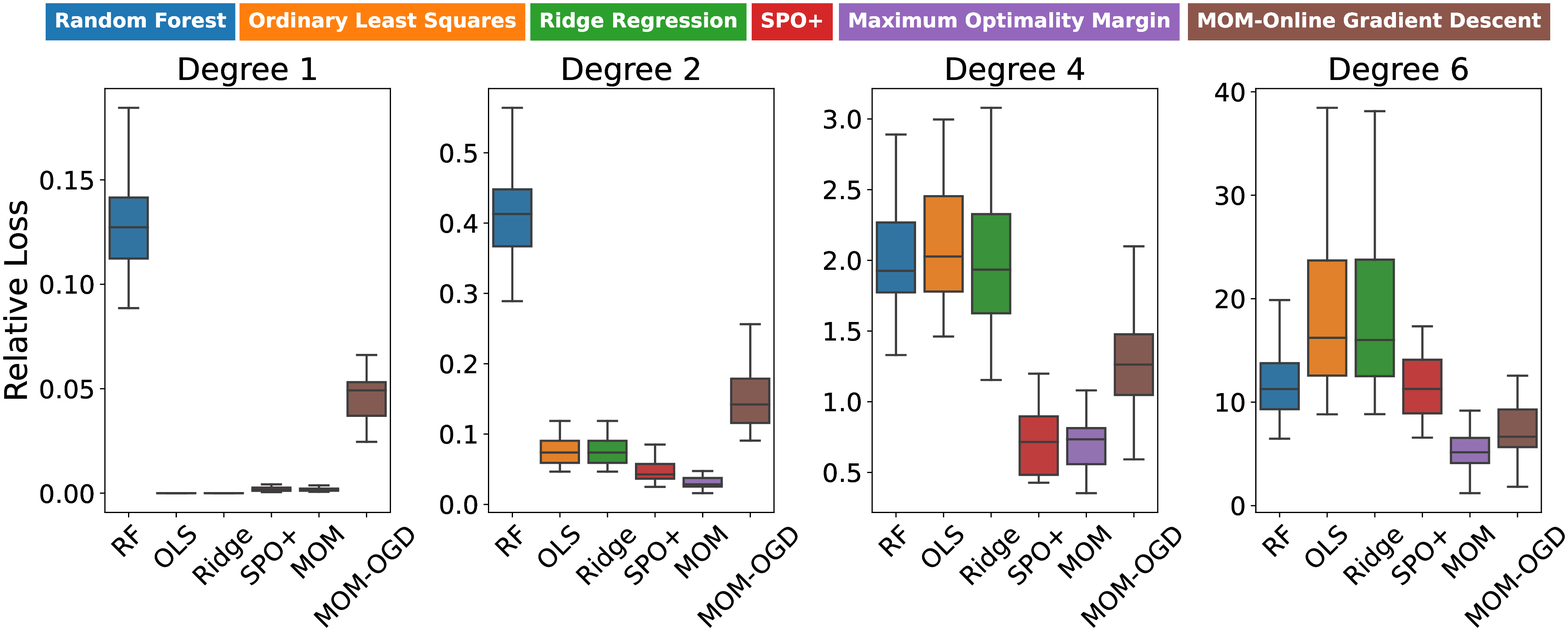}
\caption{Relative loss. The plot (means and confidence intervals) is generated based on 30 random trials each with 1000 training samples and 1000 test samples (20\% of training samples are used for tuning hyper-parameters). The degree indicates the degree of the true polynomial function that generates the objective $c$ from the covariates $z$. RF denotes random forests, OLS denotes ordinary least squares, Ridge denotes the ridge linear regression, SPO+ implements the algorithm of \cite{elmachtoub2022smart}, MOM implements Algorithm \ref{alg:MOM}, and MOM-OGD implements Algorithm \ref{alg:sgd}. For all these methods, they first predict the objective vector for a test sample, and solve an LP with the predicted objective for the optimal solution as the predicted solution for this test sample. }
\label{fig:Exp1}
\end{figure*}

\section{Numerical Experiments}
\label{sec:num_exp}

We conduct numerical experiments for two underlying LP problems -- the shortest path problem and the fractional knapsack problem. We present one experiment here and defer the remaining ones to Appendix \ref{apd:num_exp}. The experiments compare the proposed MOM algorithms against several benchmarks and illustrate different aspects, such as loss performance, computational time, scale consistency/invariance, sample complexity, kernelized version of MOM, and online setting. 

Specifically, we consider a shortest path (SP) problem here following the setup of \cite{elmachtoub2022smart}. The SP problem is defined on a $5\times 5$ grid network with $n=40$ directed edges associated with a cost vector $c\in \mathbb{R}^n$, and the covariates $z\in\mathbb{R}^d$ with $d=6$. For the training data, the covariates are generated from the Gaussian distribution, and the objective vector is generated from a polynomial mapping from the covariates with additional white noise. For this numerical experiment, we consider the performance measure of relative loss defined by
$$ \text{Relative Loss} \coloneqq \frac{c_{\text{new}}^\top(x_{\text{new}} - x^*_{\text{new}})}{c^\top x^*_{\text{new}}}$$
where $c_{\text{new}}$ is the objective vector of a new test sample, $x_{\text{new}}$ is the predicted optimal solution, and $x_{\text{new}}^*$ is the true optimal solution. Indeed, this relative loss normalizes the estimate loss $l_{\text{est}}$ with the optimal objective value.
We defer more details on the experiment setup to Appendix \ref{apd:num_exp}.

Figure \ref{fig:Exp1} presents the experiment results for two MOM algorithms and a few benchmarks, and each panel represents a different degree of the polynomial that governs the true mapping from the covariates to the objective vector; a higher degree indicates a stronger non-linearity between the covariates and the objective. We make the following observations:

First, from an information viewpoint, all four benchmark algorithms utilize the observations of $c_t$'s on the training data, while the two MOM algorithms only observe $x_t^*$'s on the training data. In this sense, our algorithms give a better predictive performance with less amount of information, and therefore our algorithms are the only ones applicable to the inverse LP problem.

Second, other than the two MOM algorithms, the SPO+ performs the best among the four benchmark algorithms. However, the SPO+ takes significantly longer training time than all the other algorithms. Although SPO+ is a stochastic gradient descent-based algorithm, the calculation of the stochastic gradient at each step requires solving $k$ LPs with $k$ being the number of training samples in the mini-batch. Comparatively, our MOM Algorithm \ref{alg:MOM} only solves a quadratic program for once, and its online version of Algorithm \ref{alg:sgd} allows a direct calculation of the online sub-gradient which makes it even faster than Algorithm \ref{alg:MOM}.

Third, the three ML-based methods (RF, OLS, Ridge) perform generally worse than the other three methods (SPO+, MOM, MOM-OGD) because they do not take into account the optimization structure. To see this, for the shortest path problem, there may be some large entries for the cost vector $c$. The ML models treat all the dimensions of $c$ equally and spend too much effort in fitting those large entries (as they cause large L$_2$ errors). However, from the optimization perspective, an accurate estimation of those large entries is not useful in that the optimal path will avoid those edges. That highlights the intuition behind the SPO+ and our MOM algorithms -- one needs to incorporate the optimization structure into the learning model.


\section{Discussions}
\label{sec:discussions}
Many existing inverse LP algorithms \citep{barmann2018online, chen2020online} implement algorithms that directly minimize the suboptimality loss $l_{\text{sub}} = \hat{c}^\top x^* - \hat{c}^\top \hat{x}$ via an online gradient descent (OGD) algorithm. We briefly showcase a simple example in the linear case $\hat{c} = \Theta z$ with $l_{\text{sub}, t}(\Theta) = \langle \Theta, (x_t^* -\hat{x}_t) z_t^\top \rangle$ and regarding the gradient as $(x^* -\hat{x}) z^\top$. One may wonder if such an OGD achieves a performance guarantee of $O(1/\sqrt{T})$ (as our result in the previous sections) following the analysis of OGD of the convex functions. But we will show its impossibility.

Why not OGD? The reason is that the minimizer of $l_{\text{sub}}(\Theta)$ is $l_{\text{sub}}(0) \equiv 0$. The dilemma of OGD starts: if the feasible region of the parameter $\mathcal{K}$ contains the original point 0, $\Theta_t$ under OGD will be gradually drawn to 0. But this means that the algorithm does not learn anything meaningful. This is a common issue for many inverse LP papers \citep{bertsimas2015data, mohajerin2018data, barmann2018online, chen2020online}, regardless of offline and online. In fact, this is what we refer to as scale inconsistency.

Alternatively, if we constrain the feasible region $\mathcal{K}$ to only contain the unit-norm $\Theta$'s, then the loss is not geodesically convex on the manifold of the unit sphere. To see this, if we impose a unit sphere constraint, then the OGD algorithm \citep{zinkevich2003online} will fail, because a basic requirement for OGD is the underlying problem's convexity. The OGD will require additional structures such as convexity over the manifold of the unit sphere, which does not hold even for linear functions \citep{jain2017non}.

Following the above arguments, one may find a seeming contradiction: the suboptimality loss $l_{\text{sub}, t}(\Theta) = \langle \Theta, (x_t^* -\hat{x}_t) z_t^\top \rangle$ seems to be a linear function of $\Theta$ at first glance, but why would the optimal $\Theta$ to be at an interior point $0 \in \mathcal{K}$ while the theory of linear programs tells us that the optimal solution of an LP always lies at the boundary? We need to notice the dependence of $\hat{x}_t$ on $\Theta$ since $x_t$ is the optimal solution induced by $\Theta z_t$. Such a dependence breaks the linearity, which means that $(x_t^* -\hat{x}_t) z_t^\top$ is even not the gradient. Should one get the gradient right, they cannot avoid being trapped by the naive prediction $\Theta = 0$ as mentioned before.

Finally, we explain the differences between the $\Theta_t \rightarrow 0$ phenomena of the above OGD and our MOM approach but under the constraints \eqref{alter2}. The cumulative suboptimality loss $\sum_{t=1}^T l_{\text{sub}, t}(\Theta)$ is actually bounded by $\sum_{t=1}^T \|s_t\|_1$ (proved in Lemma \ref{lem:optc}). The upper bound relation holds when $s_t$ satisfies either \eqref{alter2} or the original MOM constraints \eqref{prob:svm}. Such a relation has two implications: (i) the margin mechanism avoids MOM algorithms converging to zero; (ii) the performance guarantee for the MOM objective directly transfers to that of $l_{\text{sub}, t}(\Theta)$.

\textbf{Concluding remarks.} In this paper, we develop a new approach named Maximum Optimality Margin for the problems of predict-then-optimize and inverse LP. The MOM framework leads to new characterizations of the difficulty of the problem through the separability condition of Assumption \ref{assp:theta} and the violation loss $l(D_{\text{new}};\Theta)$ which appears in Propositions \ref{prop:ofgen}, \ref{prop:insep}, and \ref{prop:sgd}. With the MOM perspective, we derive bounds on correctly predicting the optimal solution,
$ \mathbb{P}(x^*_{\text{new}}=\hat{x}_{\text{new}})$
in Propositions \ref{prop:ofsprt}, \ref{prop:sgd}, and \ref{prop:pcpt}, which is rarely seen in the existing literature even under strong conditions. Without the separability condition of Assumption \ref{assp:theta}, we draw a connection of our MOM objective value with the suboptimality loss $l_{\text{sub}}$ in the inverse LP literature and derive performance bounds in terms of $l_{\text{sub}}$. Numerically, we observe the MOM algorithms also perform well under the estimate loss $l_{\text{est}}$. To theoretically quantify the performance of MOM under $l_{\text{est}}$, we conjecture that one needs to impose stronger structures on the dual LPs. Besides, another interesting future direction to pursue is to generalize the algorithms and analyses of MOM to non-linear optimization problems.

\bibliographystyle{informs2014}
\bibliography{main}

\newpage
\appendix

\section{Numerical Experiments}
\label{apd:num_exp}

In this section, we provide the details of the experiment setup and more numerical results. The codes and data can be found on \url{https://github.com/liushangnoname/Maximum-Optimality-Margin}.

\subsection{Basic setup}

\label{sec:num}

For the numerical experiments, we compare the performance of our proposed algorithms against several benchmark methods, and we test the performance under both online and offline settings. Specifically, we consider the following benchmark methods where all of them can only be used to solve the predict-then-optimize problem (but not the contextual linear programming) as they all require the observations of $c_t$ on the training data.
\begin{enumerate}
\item Linear regression models where we consider ordinary least squares (OLS) and ridge regression (RR). Accordingly, the parameter estimation follows the following loss functions:
    $$l_{\text{OLS}} = \frac{1}{2}\|\Theta z -c\|_2^2,\quad l_{\text{Ridge}} = \frac{1}{2}\|\Theta z - c\|_2^2 + \frac{\lambda}{2}\|\Theta\|_F^2.$$
\item Random Forest (RF). We apply RF algorithm with squared error criterion $$l_{\text{RF}}=\|\hat{c}-c\|_2^2.$$
\item Predict-then-optimize method. We take SPO+ algorithm \citep{elmachtoub2022smart} as an example, where the loss function is as follows
    $$ l_{\text{SPO+}} = (2\Theta z - c)^\top x^*(c) - \min_{Ax=b, x\geq0}\{(2\Theta z -c)^\top x\}. $$
Following the previous implementations, we optimize the SPO+ loss under a stochastic gradient descent (SGD) approach and Frobenius regularization.
\end{enumerate}

\textbf{Underlying LP problems.} \

We study two canonical LP problems for the numerical experiments, namely the shortest path (SP) problem and the fractional knapsack (FK) problem, and we present their results in the following two sections \ref{sec:Num-SP} and \ref{sec:Num-FK}, respectively. The experiment setup follows that of \cite{elmachtoub2022smart,ho2022risk, besbes2021contextual}.

\textbf{Performance Measurements.}

We compare different algorithms by the performance measure of \emph{relative loss} for the offline setting and \emph{cumulative Regret} in the online setting. 

Under the offline setting, the relative loss normalizes the estimate loss $l_{\text{est}}$ by the optimal value,
$$ \text{Rel-Loss}_{\text{SP}} \coloneqq \frac{c^\top(x - x^*)}{c^\top x^*}. $$
Specifically, for the shortest path (SP) problem, the optimal value is always non-zero (unless for trivial $c$) so the loss is well-defined.

For the fractional knapsack (FK) problem, due to the random noise, the optimal value may be zero with a positive probability. So we consider another normalization that leads to 
$$ \text{Rel-Loss}_{\text{FK}} \coloneqq \frac{c^\top(x^* - x)}{\|c\|_2}. $$

Under the online setting, we define cumulative regret as the cumulative sum of the relative loss of each time step.

\textbf{Overview of the results.}

In Appendix \ref{sec:Num-SP}, we demonstrate the algorithm performance under the effect of degree (see Figure \ref{fig:Exp1}) and also investigate numerically the issue of scale inconsistency/invariance (see Figure \ref{fig:Exp2}). In Appendix \ref{sec:Num-FK}, we examine the sample complexity under the offline setting (see Figure \ref{fig:Exp3}) and compare the performance of several online algorithms (see Figure \ref{fig:Exp5}). In addition, we consider the kernelized version of the MOM formulation and present its performance in Figure \ref{fig:Exp4}.

\subsection{Shortest path problem}
\label{sec:Num-SP}
For the shortest path problem, we follow the experiment setup of \citep{elmachtoub2022smart}. Consider a $5\times 5$ grid network with $n=40$ directed edges associated with a cost vector $c\in \mathbb{R}^n$. Each edge is either from south to north or west to east, where the goal is to minimize the cost to traverse from the southwest corner to the northeast corner. For each node of the network, there is a flow balance constraint, encoded by $Ax=b$. The problem can then be formulated as the standard-form LP \eqref{lp:std}.

For the methods of ordinary least square, ridge regression, and our maximum optimality margin, we solve the corresponding optimization problems using \texttt{cvxpy} package with \texttt{GUROBI} solver. 
For the random forest algorithm, we implement the \texttt{sklearn.ensemble.RandomForestRegressor} package. For the SPO+ algorithm and the online version of our MOM approach, we implement a gradient descent approach. For each simulation trial, the training sets and test sets both consist of $N=1000$ sample sizes. All regularization parameters are chosen from $10^{-6}$ to $10^2$, and all step sizes are chosen from $10^{-3}$ to $10^{1}$. All projection radius (the diameter of $\mathcal{K}$) are chosen from $0.8\|B\|_F$ to $2.0\|B\|_F$. The hyper-parameters are chosen via a validation subset of $N/4$ samples from the training data.

\emph{Synthetic Data Generation}. For each trial of our experiments, we generate a gound-truth coefficient matrix $V \in \mathbb{R}^{n\times d}$ independently, where each component of $V$ is drawn from $\text{Bernoulli}(0.5)$. Here $n=40$ and $d=6$.
\begin{enumerate}
    \item We generate the feature vectors $\{z_t\in \mathbb{R}^d\} $. The first $d-1$ components are taken independently from a standard Gaussian distribution, and the last component of each $z_t$ is set to be 1 (to model the constant in predictors).
    \item Then the true cost vectors are generated as follows:
    $$ c_{t,j} = \left[(\frac{1}{\sqrt{d}}(V z_t)_j + 3)^{\text{deg}} + 1\right]\cdot \epsilon_{t,j}\cdot \alpha_t + \bar{\eta} \eta_{t,j}, $$
    where deg is a fixed integer to be chosen to represent the nonlinearity of the problem, $\epsilon_{t,j}$ represents the intrinsic diverse noise, $\alpha_t \in \{1,1+\bar{\alpha}\}$ stands for a scale noise which depends on $z_t$, and $\eta_{t,j}$ is the additive noise. $\bar{\epsilon}$, $\bar{\alpha}$, and $\bar{\eta}$ are non-negative parameters chosen to control the power of each noise. Since the distribution of $\alpha_t$ depends on feature $z_t$, it henceforth can be considered as a type of attack that changes the scale of $c_t$. More specifically,
    $$\epsilon_{t,j} \sim \text{Unif}[1-\bar{\epsilon}, 1+\bar{\epsilon}],$$
    $$ \alpha_t = \begin{cases} 1+\bar{\alpha}, & \text{if }z_{t,1} > 0.5\\
    1, & \text{other cases},
    \end{cases}$$
    $$ 2\eta_{t,j} + 1 \sim \text{Exponential}(1). $$
\end{enumerate}

\emph{Results}. Using the above data generation process, we test the algorithms Random Forest (RF), Ordinary Least Squares (OLS), Ridge Regression (Ridge), SPO+, Maximum Optimality Margin (MOM), and MOM-OGD over 30 different trials. We test the effect of deg (see Figure \ref{fig:Exp1}) and $\bar{\alpha}$ (see Figure \ref{fig:Exp2}) in these problems, where other parameters are tested later in Fractional Knapsack problem.

The first experiment is to test the model misspecification error of different algorithms, where the deg parameter varies among $\{1,2,4,6\}$. There is no noise at all (i.e. $\bar{\epsilon} = \bar{\alpha} = \bar{\eta} = 0$) so as to emphasize the misspecification effect. The results can be found in Figure \ref{fig:Exp1}.

Another experiment is to study the effect of scale noise. As is shown in Appendix \ref{sec:consist_of_lr}, if the scale noise is independent of the feature vector, the consistency of the linear regressor still holds. So we design a scale noise that is correlated with the feature vector, which is the $\alpha_t$ in our setting. To underline the effects of scale noise, we set $\text{deg}=1$ and $\bar{\epsilon} = \bar{\eta} = 0$ in our experiments. The result is shown in Figure \ref{fig:Exp3}.

\begin{figure}[tbp]
    \centering
    \includegraphics[scale=0.35]{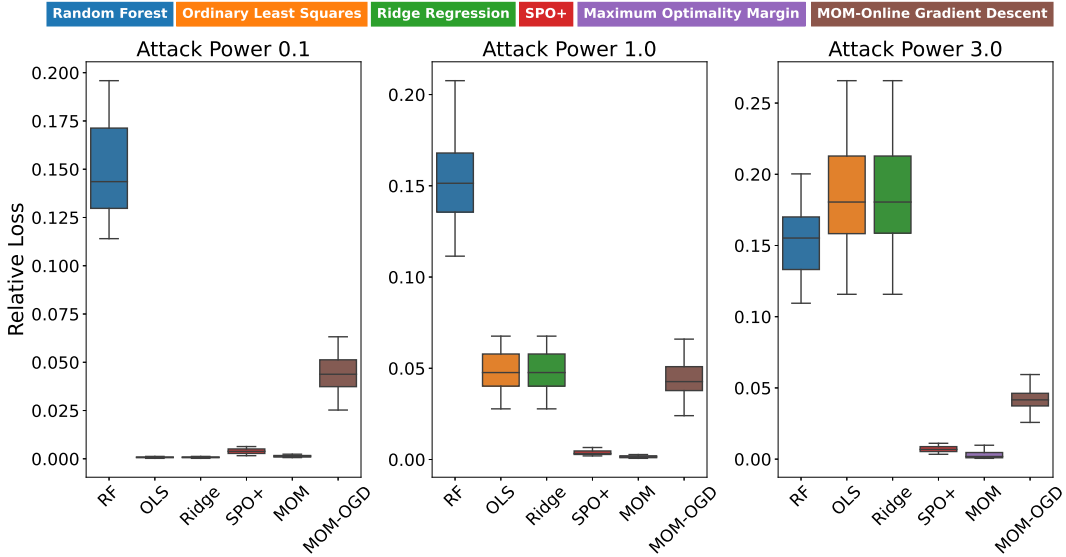}
    \caption{Experiment 2: $\text{Rel-Loss}_{\text{SP}}$ under different $\bar{\alpha}$'s, 30 trials. Here $\bar{\alpha}$ represents the power of the scale noise, denoted by \emph{Attack Power}.}
    \label{fig:Exp2}
\end{figure}

\emph{Analysis}. From the result of the first experiment (see Figure \ref{fig:Exp1}), we can see that for small degrees $\text{deg} \in \{1,2\}$ where the linear model estimation is not a heavy mistake, our algorithm MOM, SPO+ of \cite{elmachtoub2022smart}, and linear regression type methods behave similarly, where linear methods perform the best if the model is exactly linear while MOM slightly dominates other methods in the $\text{deg}=2$ case. The Random Forest algorithm performs not so well in those cases probably due to the fact that it is a non-parametric method. For high degrees $\text{deg} \in \{4,6\}$, the model misspecification error becomes rather severe. The performance of linear methods falls drastically, while the performance of MOM dominates other algorithms. The good performance of the MOM approach could possibly be explained in the following way: polynomial function with a rather high degree magnifies the components of cost vectors that are larger than 1, where those algorithms which directly predict cost vectors place equal weights on all components and their focus are dragged to those unusually large components. But for the Shortest Path problem, improving prediction accuracy on those large components hardly benefits the performance of the algorithm, since the optimal solution will avoid those edges with huge costs. Instead, the prediction accuracy on those moderate and small components is worsened by focusing on the large components, which holds back the performance of RF, OLS, and Ridge. The MOM approach behaves alternatively: for those large components of cost vectors, the estimator does not place any further concern if their corresponding reduced costs are large than 1, so the MOM estimator outperforms other methods.

Note that the SPO+ method we apply here is a stochastic gradient descent version, which empirically converges with batch size 5 in 2000 steps. Our online gradient descent version of our MOM approach can also be viewed as a stochastic gradient descent implementation of our MOM approach with batch size 1 and 1000 time steps. The performance gap between the offline MOM algorithm and the MOM-OGD algorithm can be partly explained by the small batch size and the lack of iterations.

The second experiment indicates the effects of feature-dependent scale noise, which can be found in Figure \ref{fig:Exp2}. Note that the scale noise is only added to data when the first component of the feature is larger than 0.5. As a typical ensemble method, RF performs quite steadily due to the nature of tree estimators. That is not the case for linear regressors, since their estimation of true coefficients of the first feature component will be heavily disturbed. For the other three methods MOM, SPO+, and MOM-OGD, their performances are quite stable, thanks to the fact that they are only concerned with the optimal solutions of linear programs that remain unchanged under scale noises.

\subsection{Fractional knapsack problem}
\label{sec:Num-FK}

For this experiment, we follow the problem setup 
of \citep{ho2022risk}. Specifically, each decision maker is presented with a bunch of items, and their aim is to maximize the utility (or equivalently, to minimize the cost) under a knapsack constraint. Different from the discrete case, the fractional knapsack problem allows the decision maker to select a fraction of some certain item, making it a linear program. Mathematically, the decision-maker solves the following LP
$$ \min_{x,s_1,s_2}\  c^\top x, \quad \text{s.t. } p^\top x + s_1 = B,\  x+s_2=1,\  x,s_1,s_2 \geq 0. $$
Here $p \in \mathbb{R}^n$ can be interpreted as the price vector and $B \in \mathbb{R}$ is the budget of the decision maker. The corresponding utility vector should be interpreted as the opposite of the objective vector $c$ since the LP considers a minimization problem. The slack variables $s_1$ and $s_2$ are introduced simply to convert the problem into a standard form.

As for the offline learning problems, the algorithms we implement are exactly the same as we do in \ref{sec:Num-SP}. We also apply the online setting here with three algorithms: Perceptron (as is illustrated in Algorithm \ref{alg:prcpt}), Maximum Optimality Margin with Follow the Regularized Leader (as is illustrated in Algorithm \ref{alg:ftl}), and Maximum Optimality Margin with Online Gradient Descent (as is illustrated in Algorithm \ref{alg:sgd}).

\emph{Synthetic Data Generation}. At the beginning of every trial, we generate a ground-truth coefficient matrix $V \in \mathbb{R}^{n \times d}$, where each component is a Bernoulli random variable with an expectation of 0.5. Here $n=10$ and $d=5$. We then generate the price vector, where each component is selected to be a random integer between 1 and 1000 in an offline setting. To ease the pain of large constants, we re-scale the price vector in an online setting, where each component is uniformly distributed between 0 and 1. We generate two auxiliary variables \emph{low} and \emph{high}, where $\text{low} = \max_{j} p_j$ and $\text{high} = \bm{1}^\top p - u \cdot \text{low}$, and $u \sim \text{Unif}[0,1]$. Then the budget is chosen as $\text{Unif}[\text{low}, \text{high}]$.
\begin{enumerate}
    \item We generate the feature vectors $\{z_t \in \mathbb{R}^d\}$. The first $d-1$ components are taken independently from a $\text{Unif}[0,1]$ distribution and the last component of each $z_t$ is set to be 1 (to model the constant in predictors).
    \item Then the true cost vectors are generated as follows:
    $$ c_{t,j} = (V z_t)_j^{\text{deg}}\cdot \epsilon_{t,j}\cdot \alpha_t + \bar{\eta} \eta_{t,j}, $$
    where deg is a fixed integer to be chosen to represent the nonlinearity of the problem, $\epsilon_{t,j}$ represents the intrinsic diverse noise, $\alpha_t \in \{1,1+\bar{\alpha}\}$ stands for a scale noise which depends on $z_t$, and $\eta_{t,j}$ is the additive noise. $\bar{\epsilon}$, $\bar{\alpha}$, and $\bar{\eta}$ are non-negative parameters chosen to control the power of each noise. Since the distribution of $\alpha_t$ depends on feature $z_t$, it henceforth can be considered as a type of attack that changes the scale of $c_t$. More specifically,
    $$\epsilon_{t,j} \sim \text{Unif}[1-\bar{\epsilon}, 1+\bar{\epsilon}],$$
    $$ \alpha_t = \begin{cases} 1+\bar{\alpha}, & \text{if }z_{t,1} > 0.5\\
    1, & \text{other cases},
    \end{cases}$$
    $$ 2\eta_{t,j} + 1 \sim \text{Exponential}(1). $$
\end{enumerate}


\emph{Results}. Equipped with such a data generation process, we apply several numerical experiments in both offline and online settings. We first test the sample complexity under offline setting, with $\text{deg} = 1$, $\bar{\epsilon} = 0.1$, $\bar{\eta} = 1.0$, and $\bar{\alpha} = 0.0$. We apply 30 trials. For training sets, we generate different $T=100,200,500,1000,5000$, and test the performance against 1000 test samples. The algorithms we choose contain Random Forest (RF), Ordinary Least Squares (OLS), Ridge Regression (Ridge), SPO+ of \cite{elmachtoub2022smart}, our Maximum Optimality Margin Algorithm \ref{alg:MOM} (MOM), and our MOM-OGD Algorithm \ref{alg:sgd}. All regularizing constants are chosen from $10^{-6}$ to $10^2$. All step sizes are chosen from $10^{-3}$ to $10^{1}$. All projection radii are chosen from $0.8\|B\|_F$ to $2.0\|B\|_F$. The hyper-parameters are decided via a validation set of $T/4$ samples and the criterion of averaged Relative Loss. The results can be found in Figure \ref{fig:Exp3}.

\paragraph{Sample complexity.}\

\begin{figure}[ht!]
    \centering
    \includegraphics[scale=0.5]{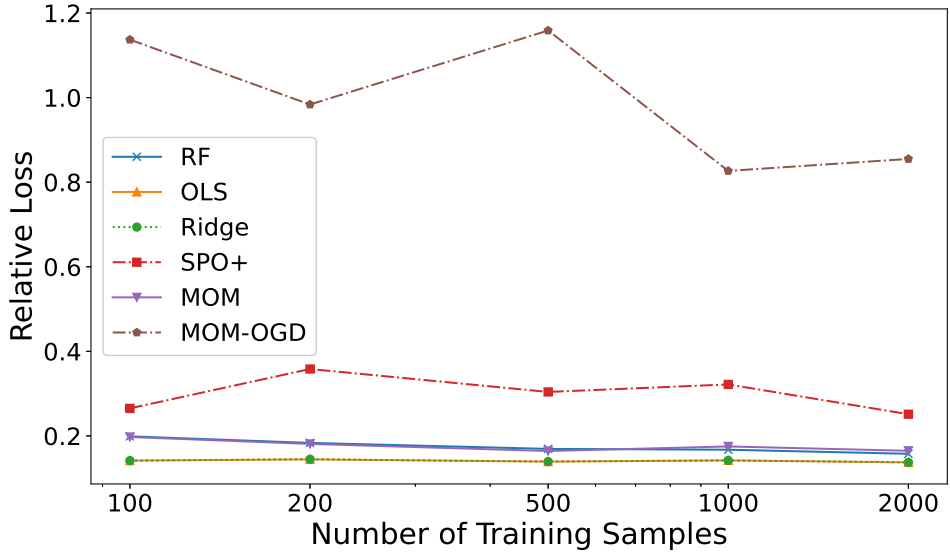}
    \caption{Sample complexity of different algorithms.}
    \label{fig:Exp3}
\end{figure}

\paragraph{Kernelization and non-linear MOM.}\

We then develop some kernelized versions of our MOM approach, resulting in three different SVM-based algorithms: Linear MOM (Linear), Polynomial Kernelized MOM (PolyKer), Radial Basis Function Kernelized MOM (RbfKer). For those kernelized methods, the original feature $z_{t_0}$ is replaced by an $T$-dimensional feature $(\kappa(z_{t_0}, z_t))_{t=1}^T$, where $\{z_t\}_{t=1}^T$ is the training set of features and $\kappa$ is some pre-chosen kernel function. For the polynomial kernel, we utilize $\kappa_{\gamma,d_0}(z_i, z_j) = (z_i^\top z_j \gamma^{-1} + 1)^{d_0}$. For the radial basis function kernel, we define $\kappa_{\gamma}(z_i, z_j) = \exp(-\frac{\|z_i-z_j\|^2}{\gamma})$. The degree of the polynomial kernelized MOM is chosen from $\{1,2,3,4\}$ and the scale parameters $\gamma$'s of both kernelized MOM methods are chosen from $\{0.1, 0.5, 1, 2, 3, 4, 5\}$. All regularizing constants are chosen from $10^{-6}$ to $10^2$. The training data size is now reduced to $N=500$. The hyper-parameters are decided via a validation set of $T/4$ samples and the criterion of averaged Relative Loss. We paint the boxplots of 30 independent trials. The results can be found in Figure \ref{fig:Exp4}.

\begin{figure}[ht!]
    \centering
    \includegraphics[scale=0.4]{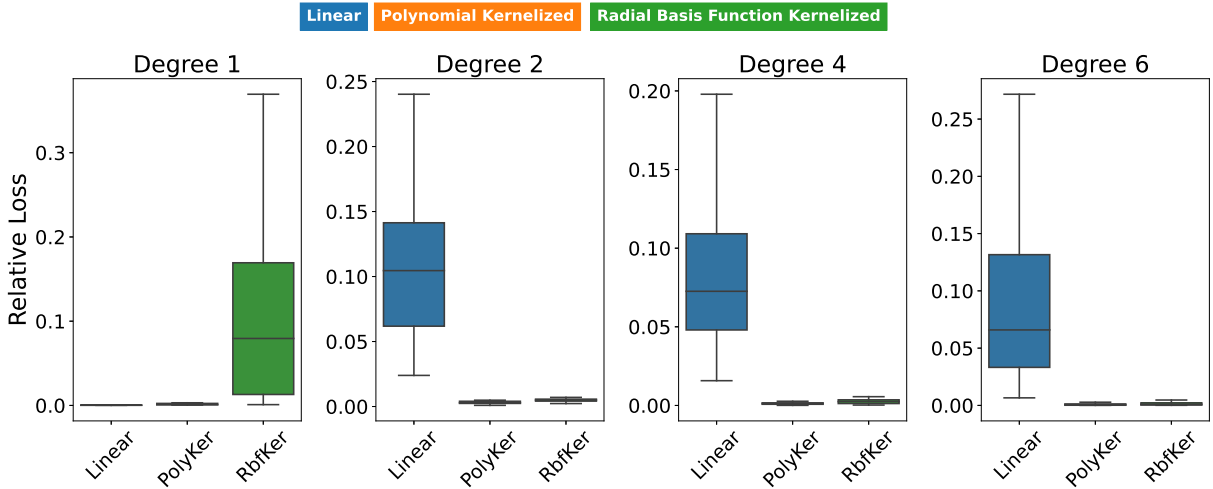}
    \caption{Experiment 4: $\text{Rel-Loss}_{\text{FK}}$ under different data generation degrees, 30 trials.}
    \label{fig:Exp4}
\end{figure}

The fourth experiment is about the kernelized versions of our MOM approach, where three MOM algorithms named Linear, Polynomial Kernelized, and Radial Basis Function Kernelized are implemented (see Figure \ref{fig:Exp4}). The performance of linear MOM is worsened as the degree of the model grows since the model misspecification becomes more severe. But kernelized MOM methods remains good performance under high degrees. Note that the performance of RBF Kernelized MOM's behavior is unsatisfying when the model is perfectly linear. Those kernelized versions of our algorithms reveal the potential of generalizing our MOM principle to other cases apart from just the linear model.

\emph{Analysis}. The third experiment we adopt is examining the sample complexity of the aforementioned algorithms under a noisy environment (see Figure \ref{fig:Exp3}). The performance of linear regressors exceeds other algorithms since there is no model misspecification at all. The performance of the Random Forest algorithm is similar to that of the offline Maximum Optimality Margin algorithm, which is slightly worse than that of linear regressors. As for the SPO+ method, their averaged performance is undermined by the noise to a fairly large extent, compared with its performance under a noiseless environment in Experiment 1 and Experiment 2 (see Figure \ref{fig:Exp1} and Figure \ref{fig:Exp2}). The difference in the performance between SPO+ and MOM can be partially explained by the different ways of dealing with noisy data. SPO+ directly put the optimal solution of the noisy data to the gradient, while the optimal solution could probably vary exaggeratedly due to some noise to the cost vector. MOM acts more steadily: for those noisy cost vectors, their optimal solutions could be far from the noiseless cost vectors, but the change with respect to optimal basis will usually be only a few components which only affects a few corresponding lines in estimated $\Theta$. Finally, we note that the small batch size and the lack of iteration could be blamed for the poor performance of MOM-OGD in noisy data since the noisy data exacerbates the variance.

To complete the argument, we provide some extra numerical experiments to compare our kernelized methods with other kernelized methods such as kernelized ridge regression. The first extra experiment adopts the same setting as that in Figure \ref{fig:Exp4}. We keep the same setup (except for only $10$ independent trials instead of $30$ trials to save the time cost) again for 6 benchmark algorithms: Random Forest (RF), Ordinary Least Squares (OLS), Ridge Regression (Ridge), Polynomial Kernelized Ridge Regression (PolyRidge), Rbf Kernelized Ridge Regression (RbfRidge), and SPO+, and 3 of our MOM algorithms: Maximum Optimality Margin (MOM), Polynomial Kernelized MOM (PolyMOM), and Rbf Kernelized MOM (RbfMOM). The results can be found in Figure \ref{fig:Exp6}.

\begin{figure}[ht!]
    \centering
    \includegraphics[scale=0.45]{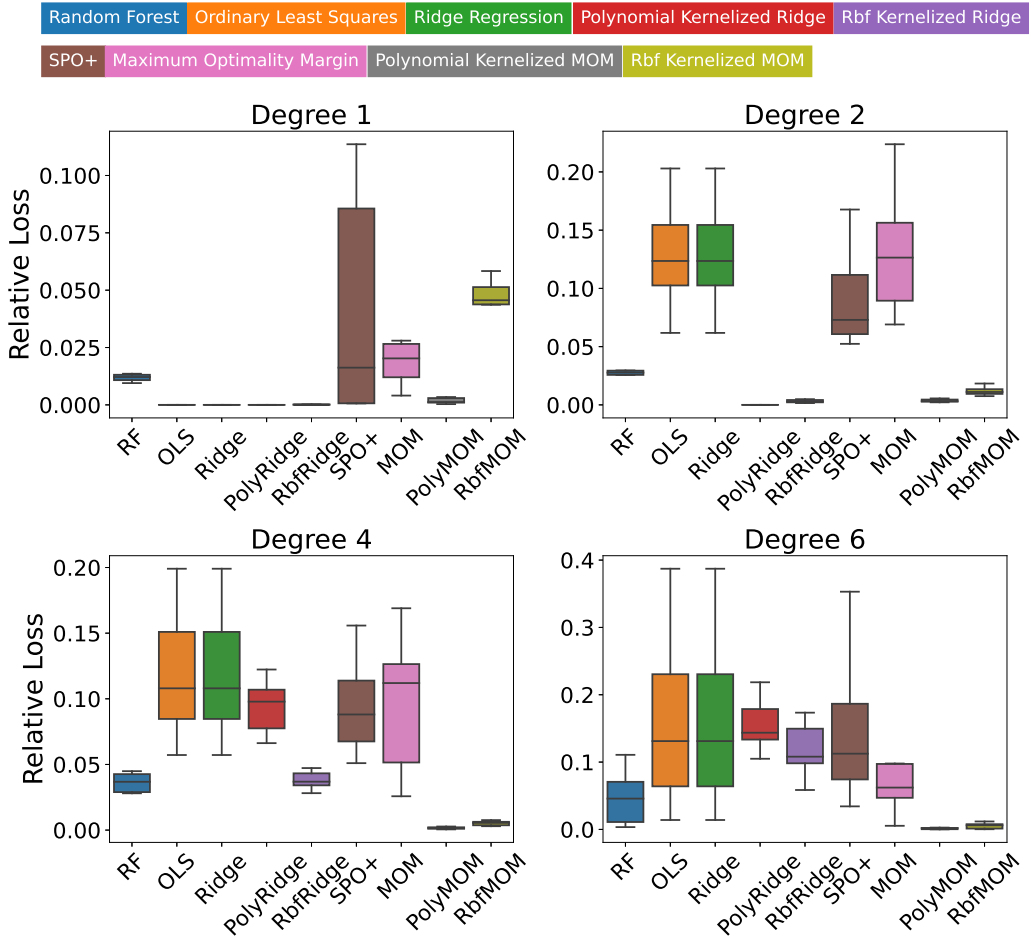}
    \caption{An extra comparison of the relative loss for kernelized MOMs and benchmark algorithms for the Fractional Knapsack problem.}
    \label{fig:Exp6}
\end{figure}

We also provide another result under a similar setting for the Shortest Path problem. The training sample size is now reduced to $200$, and the number of independent trials is also reduced to $10$ to ease the computational price. To emphasize the scale consistency of our MOM approach, we apply the same noise setup as the last sub-figure of Figure \ref{fig:Exp2}. The results are summarized in Figure \ref{fig:Exp7}.

\begin{figure}[ht!]
    \centering
    \includegraphics[scale=0.45]{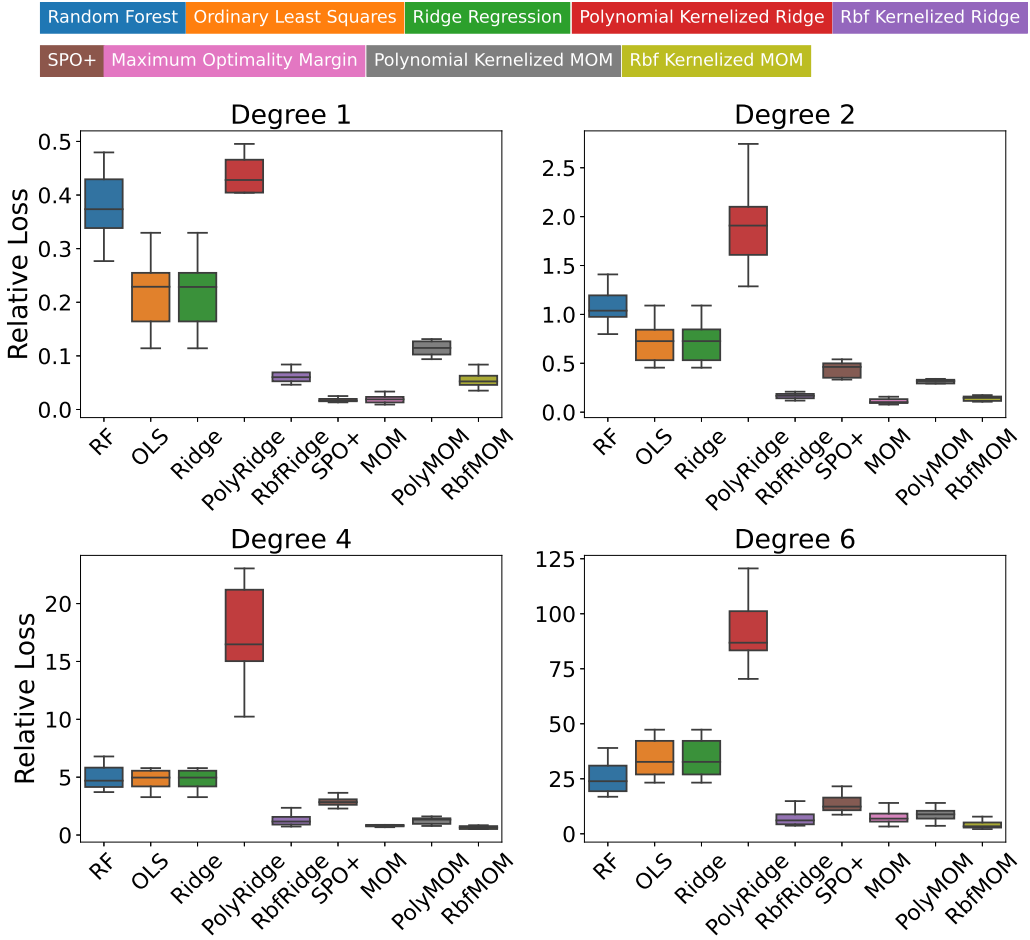}
    \caption{Relative loss for kernelized MOMs and benchmark algorithms for the Shortest Path problem.}
    \label{fig:Exp7}
\end{figure}

\subsection{Online setting}

\label{sec_online_experi}

Next we analyze the online setting for the MOM algorithms. We compare the performance of the OGD version of MOM (Algorithm \ref{alg:sgd}), the perceptron version of MOM (Algorithm \ref{alg:prcpt}), and also a follow-the-regularized-leader (FTRL) version of the MOM Algorithm \ref{alg:MOM} as follows. Basically, Algorithm \ref{alg:ftl} solves the MOM optimization formulation \ref{prob:svm} at each time $t$ and use the estimated parameter to predict the optimal solution of the next time step.

\begin{algorithm}[ht!]
    \caption{Follow-the-Regularized-Leader MOM}
    \label{alg:ftl}
    \begin{algorithmic}[1] 
    \Require Dataset $\mathcal{D}(T)=\{( {x}_t^*, {A}_t, {b}_t,z_t,\mathcal{B}_t^*,\mathcal{N}_t^*)\}_{t=1}^{T}$, the set $\mathcal{K}$
    \State Initialize ${ {\Theta}}_1 =  {0}\in\mathbb{R}^{n\times d}$
    \For {$t=1,...,T$}
\State  Predict the objective by $\hat{{c}}_t= {\Theta}_{t} {z}_t$
            \State  Solve the following LP and denote its optimal solution by $x_t$ \begin{align*}
 \min \ &  \hat{c}_{t}^\top x,\\
    \text{s.t.\ } &  A_{t}x=b_{t}, \ x\ge 0. \nonumber
\end{align*}
        \State Solve the optimization problem \eqref{prob:svm} with $\{( {x}_s^*, {A}_s, {b}_s,z_s,\mathcal{B}_s^*,\mathcal{N}_s^*)\}_{s=1}^{t}$, $\lambda=1/\sqrt{t}$ and $\mathcal{K}$
        \State Let $\Theta_{t+1}$ be the optimal solution 
    \EndFor
    \Ensure $\{{x}_t\}_{t=1}^{T}$ 
    \end{algorithmic}
\end{algorithm}

Algorithm \ref{alg:ftl} gives another interpretation of the MOM formulation. Under an online setting, if we solve the MOM optimization problem \eqref{prob:svm} at each time, this is equivalent to a follow-the-regularized online algorithm to solve the problem. The theoretical analysis of Algorithm \ref{alg:ftl} can be extended from Proposition \ref{prop:ofsprt} and Proposition \ref{prop:ofgen}.

\begin{figure}[ht!]
    \centering
    \includegraphics[scale=0.4]{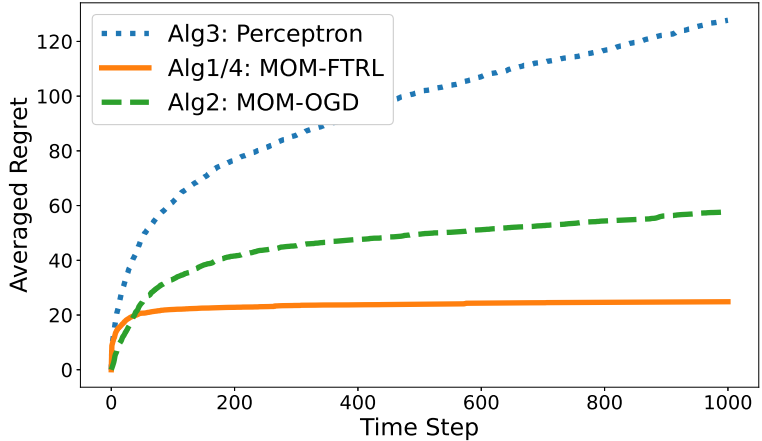}
    \caption{Cumulative regret curve (averaged over 30 trials) for online algorithms.}
    \label{fig:Exp5}
\end{figure}

Figure \ref{fig:Exp5} presents the cumulative regret of several algorithms for a fractional knapsack problem. In this numerical experiment, we let the objective vector $c=\Theta^* z$ almost surely for some fixed $\Theta^*$ so as to achieve the separability condition. We emphasize that it only ensures the existence of a parameter $\Theta^*$ to meet the optimality condition \eqref{lp_opt_con} but not with a margin of 1 as Assumption \ref{assp:theta}. From the plot, we observe that Algorithm \ref{alg:prcpt} has the worst performance though it features for the best theoretical dependency on $T$. We remark that this regret curve does not contradict the bounded number of mistakes in Proposition \ref{prop:pcpt}; we find that when we increase the horizon to $T=100,000$, the regret curve of Algorithm \ref{alg:prcpt} becomes flattened. For Algorithm \ref{alg:sgd} and Algorithm \ref{alg:ftl}, both achieve significantly better performance. Comparatively, \ref{alg:ftl} performs better than Algorithm \ref{alg:sgd}, with the price of more computation cost (to solve a quadratic program at each time period). 

Importantly, we also implement the online algorithm of \citep{barmann2018online, chen2020online} under the same setting, and it incurs a regret linearly increasing with $T$; so we do not include its regret curve as it will clap all the regret curves of Figure \ref{fig:Exp5} to x-axis. A closer investigation shows that the online algorithm of \citep{barmann2018online, chen2020online} will render the estimate $\Theta_t\rightarrow 0$, as noted by the scale inconsistency issue in earlier sections. In contrast, our margin-based formulation plays an important rule to ensure that a small optimality condition violation is caused by discovering the correct mapping from $z$ to $c$ but not by the scale of the predicted $\hat{c}.$

\section{Proofs}

\subsection{Proof of Lemma \ref{lem:opt_con}}

Here we show a stronger version of Lemma \ref{lem:opt_con} as follows.

\begin{lemma}
\label{lem:optc}
    Consider an LP of the standard form \eqref{lp:std}. For any feasible basis $\mathcal{B}\subset [n]$ satisfying $|\mathcal{B}|=m$ and its complement $\mathcal{N}=[n]\backslash \mathcal{B}$, denote $ {x}=(x_1,...,x_n)^{\top}$ and $ {r}=(r_1,...,r_n)^{\top}$ as the solution and reduced cost vector corresponding the basis $\mathcal{B}$, respectively, i.e.,
    \begin{align*}
         {x}_{\mathcal{B}} =   {A}_{\mathcal{B}}^{-1} {b},\   {x}_{\mathcal{N}} =  {0},\   
         {r} =  {c} -  {A}^{\top}( {A}_{\mathcal{B}}^{-1})^{\top} {c}_{\mathcal{B}}.
    \end{align*}
    Denote $ {x}^*=(x_1^*,...,x_n^*)^{\top}$ as one optimal solution of LP \eqref{lp:std}. Then, we have that $ {r}_{\mathcal{B}}= {0}$, and 
    \begin{align}
        \label{ieq:rc}
         {c}^{\top} {x}- {c}^{\top} {x}^*\leq
        \max_{i\in[n]} x_i^*\cdot\sum\limits_{i\in\mathcal{N}} (-r_i)_{+}.
    \end{align}
    Furthermore, \eqref{ieq:rc} implies that ${x}$ is an optimal solution if $ {r}\geq {0}$.
\end{lemma}

\begin{proof}
First, it is easy to verify that $ {r}_{\mathcal{B}}= {0}$. Then, we only need to show that the inequality \eqref{ieq:rc} holds.
    
    For any feasible solution $ {x}'=(x_1',...,x_n')\geq {0}$ satisfying $ {A} {x}'= {b}$, we have 
    \begin{align*}
         {x}_{\mathcal{B}}' =  {A}_{\mathcal{B}}^{-1} {b}- {A}_{\mathcal{B}}^{-1} {A}_{\mathcal{N}} {x}_{\mathcal{N}}'.
    \end{align*}
    It implies
    \begin{align}
    \label{eq:rcobj}
         {c}^{\top} {x}'
        =
         {c}^{\top}_{\mathcal{B}} {A}_{\mathcal{B}}^{-1} {b}
        -
         {c}^{\top}_{\mathcal{B}} {A}_{\mathcal{B}}^{-1} {A}_{\mathcal{N}} {x}_{\mathcal{N}}'+ {c}^{\top}_{\mathcal{N}} {x}_{\mathcal{N}}'
        =
         {c}^{\top}_{\mathcal{B}} {A}_{\mathcal{B}}^{-1} {b} +  {r}^{\top}_{\mathcal{N}} {x}_{\mathcal{N}}'.
    \end{align}
   Next, from \eqref{eq:rcobj}, we have
    \begin{align*}
         {c}^{\top} {x}- {c}^{\top} {x}'
        &=
         -{r}^{\top}_{\mathcal{N}} {x}'_{\mathcal{N}}\\
        &\leq
        \sum\limits_{i\in\mathcal{N}} x_i'(-r_i)_{+} \\
        &\leq
        \max_{i\in[n]} x_i'\cdot \sum\limits_{i\in\mathcal{N}} (-r_i)_{+},
    \end{align*}
    where the first line comes from equality \eqref{eq:rcobj} directly, and the last two lines come from the non-negativity of $ {x}'$ and $(r_i)_+$ for all $i$. Note that the above inequality holds for any feasible solution $x'$; by plugging in $ {x}'= {x}^*$, we obtain \eqref{ieq:rc}.
\end{proof}

\subsection{Proof of Proposition \ref{prop:ofgen}}

\label{pf:gen}

We first show Proposition \ref{prop:ofgen} and then utilize the result for the proof of \ref{prop:ofsprt}. The key is to utilize the algorithm stability analysis where we cite the result from \citep{shalev2010learnability} as the following proposition. We refer to   \citep{bousquet2002stability,shalev2010learnability,feldman2019high} for more related analysis such as the high probability bound. 

\begin{theorem}[Theorem 3 in \cite{shalev2010learnability}]
\label{thm:stb}
Let $f: \mathcal{H} \times\mathcal{Z} \rightarrow R$ be such that $\mathcal{H}$ is bounded by $B$ and $f( {h},  {z})$ is convex and $L$-Lipschitz with
respect to $ {h}$. Let $, {z}_{0}, {z}_1,...,  {z}_{T}$ be i.i.d. samples and let
\begin{align*}
    \hat{ {h}}_{\lambda} = 
    \arg\min_{ {h}\in\mathcal{H}}
    \left(
        \sum\limits_{t=1}^{T}f( {h}, {z}_t)+\frac{\lambda}{2}\| {h}\|_2^2
    \right).
\end{align*}
Then, we have
$$
    \mathbb{E}\left[ f(\hat{ {h}}_{\lambda}, {z}_0) \right]
    \leq
    \inf_{ {h}\in\mathcal{H}}\mathbb{E}\left[ f( {h}, {z}_0) \right]+\frac{\lambda}{2}B^2
    +\frac{4(L+\lambda B)^2}{\lambda T}.
$$
\end{theorem}

\begin{proof}
    We refer to Theorem 2 and Theorem 3 in \cite{shalev2010learnability}.
\end{proof}

Now we show Proposition \ref{prop:ofgen} with Theorem \ref{thm:stb}.

\begin{proof}
Recall that $$D_{\text{new}} = (c,A,b,z, \mathcal{B}^*,\mathcal{N}^*)$$ denotes a new sample from $\mathcal{P}$, and the loss function
$$l(D_{\text{new}};\Theta) \coloneqq \|s\|_1 \text{ \ where \ } s \coloneqq (1_{|\mathcal{N}^*|} - \hat{c}_{\mathcal{N}^*}^\top - \hat{c}_{\mathcal{B}^*}^\top A_{\mathcal{B}^*}^{-1} A_{\mathcal{N}^*})_{+}.  $$
Here, $(\cdot)_+$ denotes the entry-wise positive part function, and $\mathcal{B}^*$ and $\mathcal{N}^*$ are defined as in Section \ref{sec:msa}. Here, with a slight abuse of notation, we drop the sample tuple $D_{new}$ and let $l( {\Theta})=l\left( D_{\text{new}};{\Theta}\right)$. Then, under Assumption \ref{assp:bdd}, $l( {\Theta})$ is $(2+\sqrt{m}\bar{\sigma})$-Lipschitz with respect to $\Theta$ in the Frobenius norm for a fixed $D_{\text{new}}$. To see this, for any two parameters ${\Theta}=( {\Theta}_1,..., {\Theta}_n)^{\top},\hat{ {\Theta}}=(\hat{ {\Theta}}_1,...,\hat{ {\Theta}}_n)^{\top}\in\mathbb{R}^{n\times d}$
    \begin{align*}
       |l( {\Theta})-l( {\Theta}')|
       &=
       \left|\sum\limits_{i\in\mathcal{N}}
       \left((1- {\Theta}_i^{\top} {z}+ {A}_i^{\top}( {A}_{\mathcal{B}^*}^{-1})^{\top} {\Theta}_{\mathcal{B}^*} {z})_{+}-
       (1-\hat{ {\Theta}}_i^{\top} {z}+ {A}_i^{\top}( {A}_{\mathcal{B}^*}^{-1})^{\top}\hat{ {\Theta}}_{\mathcal{B}^*} {z})_{+}\right)\right|\\
       &\leq
       \sum\limits_{i\in\mathcal{N}}
       \left|(1- {\Theta}_i^{\top} {z}+ {A}_i^{\top}( {A}_{\mathcal{B}^*}^{-1})^{\top} {\Theta}_{\mathcal{B}^*} {z})_{+}-
       (1-\hat{ {\Theta}}_i^{\top} {z}+ {A}_i^{\top}( {A}_{\mathcal{B}^*}^{-1})^{\top}\hat{ {\Theta}}_{\mathcal{B}^*} {z})_{+}\right|\\
       &\leq
       \sum\limits_{i\in\mathcal{N}}
       \left|- {\Theta}_i^{\top} {z}+ {A}_i^{\top}( {A}_{\mathcal{B}^*}^{-1})^{\top} {\Theta}_{\mathcal{B}^*} {z}+\hat{ {\Theta}}_i^{\top} {z}- {A}_i^{\top}( {A}_{\mathcal{B}^*}^{-1})^{\top}\hat{ {\Theta}}_{\mathcal{B}^*} {z}\right|\\
       &\leq
       \sum\limits_{i\in\mathcal{N}}
       \left(\left|( {\Theta}_i-\hat{ {\Theta}}_i)^{\top} {z}\right|+\left| {A}_i^{\top}( {A}_{\mathcal{B}^*}^{-1})^{\top}( {\Theta}_{\mathcal{B}^*}-\hat{ {\Theta}}_{\mathcal{B}^*}) {z}\right|\right)\\
       &\leq
       \sum\limits_{i\in\mathcal{N}}\left(\left\| {\Theta}_i-\hat{ {\Theta}}_i\right\|_2\left\| {z}\right\|_2+ \sigma_{\max}\left(( {A}_{\mathcal{B}^*}^{-1})^{\top}( {\Theta}_{\mathcal{B}^*}-\hat{ {\Theta}}_{\mathcal{B}^*})\right)\left\| {A}_i\right\|_2\left\| {z}\right\|_2\right)\\
       &\leq
       2\| {\Theta}-\hat{\Theta}\|_F+\sqrt{m} \sigma_{\max}\left(( {A}_{\mathcal{B}^*}^{-1})^{\top}\right) \sigma_{\max}\left( {\Theta}_{\mathcal{B}^*}-\hat{ {\Theta}}_{\mathcal{B}^*}\right)\\
       &\leq
       (2+\sqrt{m}\bar{\sigma})\left\| {\Theta}-\hat{ {\Theta}}\right\|_F.
    \end{align*}
Here the first line comes from the definition of $l(\cdot)$, the second and fourth lines are obtained by the convexity of the absolute value function, the third line comes from a direct computation of the positive part function, the fifth line is obtained by Cauchy's inequality and the definition of $ \sigma_{\max}$, the sixth line comes from the inequality that $$ \sigma_{\max}( {X} {Y})\leq \sigma_{\max}( {X}) \sigma_{\max}( {Y})$$ for any two matrices $ {X}, {Y}$, and the last line comes from Assumption \ref{assp:bdd} and the inequality $ \sigma_{\max}( {X})\leq\| {X}\|_F$ for any matrix $ {X}$.
    
    By Theorem \ref{thm:stb} and Assumption \ref{assp:bdd}, we have
    \begin{align}
        \label{ieq:gen}
        \mathbb{E}[l(\hat{\Theta})]
        \leq
        \min_{ {\Theta}\in\mathcal{K}}\mathbb{E}[l( {\Theta})]
        +
        \frac{\lambda}{2}\bar{\Theta}^2+\frac{4(2+\sqrt{m}\bar{\sigma}+\lambda \bar{\Theta})^2}{\lambda T},
    \end{align}
    where
    $$
       \hat{\Theta}
        =
        \arg\min_{ {\Theta}\in\mathcal{K}}
        \left(
            \sum\limits_{t=1}^{T}l( D_{new},{\Theta})+\frac{\lambda}{2}\| {\Theta}\|_F
        \right).
    $$
    Finally, plugging $\lambda=\frac{1}{\sqrt{T}}$ into \eqref{ieq:gen}, we have
    $$
        \mathbb{E}[l(\hat{\Theta})]
        \leq
        \min_{ {\Theta}\in\mathcal{K}}\mathbb{E}[l( {\Theta})]
        +\frac{28+8\bar{\Theta}^2+7m\bar{\sigma}^2}{\sqrt{T}}.
    $$
\end{proof}

\subsection{Proof of Proposition \ref{prop:ofsprt}}
\begin{proof}
    Under Assumption \ref{assp:theta}, there exists a $ {\Theta}^*\in\mathcal{K}$ such that the following inequality holds almost surely
    $$
         {\Theta}^*_{\mathcal{N}^*} {z}- {A}_{\mathcal{N}^*}^{\top}( {A}_{\mathcal{B}^*}^{-1})^{\top} {\Theta}^*_{\mathcal{B}^*} {z}
        \geq
        1.
    $$ 
Thus we have 
    $$
        \min_{ {\Theta}\in\mathcal{K}}\mathbb{E}[l( {\Theta})]=0,
    $$
    where $l( {\Theta})$ is defined following the previous proof of Proposition \ref{prop:ofgen}. Then, from Proposition \ref{prop:ofgen},
    \begin{align}
        \label{ieq:genfs}
        \mathbb{E}[l(\hat{\Theta})]\leq
        \frac{28+8\bar{\Theta}^2+7m\bar{\sigma}^2}{\sqrt{T}}.
    \end{align}
For a new sample $(c_{\text{new}},A_{\text{new}},b_{\text{new}},z_{\text{new}})$ with $\hat{c}_{\text{new}}=\hat{\Theta} z_{\text{new}}$, we can utilize Lemma \ref{lem:optc} and bound the suboptimality loss as follows
    \begin{align}
        \mathbb{E}\left[\hat{c}_{\text{new}}^\top x^*_{\text{new}}-\hat{c}_{\text{new}}^\top \hat{x}_{\text{new}}\right]
        &\leq
        \mathbb{E}\left[\max_{i\in[n]}(\hat{x}_{\text{new}})_i\cdot \sum\limits_{i\in\mathcal{N}^*_{\text{new}}}(-r_i)_+\right]\nonumber\\
        &\leq
         \mathbb{E}\left[ \sum\limits_{i\in\mathcal{N}^*_{\text{new}}}(-r_i)_+\right]\nonumber\\
        &\leq
         \mathbb{E}\left[ \sum\limits_{i\in\mathcal{N}^*_{\text{new}}}(1-r_i)_+\right]\nonumber\\
        &\leq
        \frac{28 +8 \bar{\Theta}^2+7m \bar{\sigma}^2}{\sqrt{T}}.\label{tmp_prob}
    \end{align}
Here the first line comes from Lemma \ref{lem:optc}, the second line comes from Assumption \ref{assp:bdd}, the third line comes from the monotonicity of the positive part function, and the last line comes from \eqref{ieq:genfs}. 
    
Now, we note that if $\hat{\Theta}$ renders any non-basic variables in $\mathcal{N}^*$ as basic variables, i.e., $$\hat{c}_{\text{new},\mathcal{N}^*_{\text{new}}}^\top - \hat{c}_{\text{new},\mathcal{B}^*_{\text{new}}}^\top A_{\text{new},\mathcal{B}^*_{\text{new}}}^{-1} A_{\text{new},\mathcal{N}^*_{\text{new}}} \ge 0$$
does not hold,
we have $l(D_{new};\hat{\Theta})\geq 1$. Thus, by applying Markov's inequality to , we have that with probability no less than $1-\frac{28+8\bar{\Theta}^2+7m\bar{\sigma}^2}{\sqrt{T}}$, $\hat{\Theta}$ can identify both the true optimal basis and the true optimal solution correctly.
\end{proof}

\subsection{Proof of Proposition \ref{prop:sgd}}
The proof in this part is basically an application of the following lemma.
\begin{lemma}[Theorem 3.1 in \cite{hazan2016introduction}]
    \label{lem:sgdconv}
    Let $\{f_t(x)\}_{t=1}^{T}$ be a sequence of convex functions defined on $\{x:\|x\|_2\leq K\}$. Suppose $\|\nabla f_t(x)\|_2\leq G$ for all $x$ such that $\|x\|_2\leq K$ and all $t=1,...,T$. Let ${x}_{t+1}={ {x}_t}-\frac{2K}{G\sqrt{t}}\nabla f_t(x_t)$. Then, the following inequality holds
    $$
        \sum\limits_{t=1}^{T}f_t(x_t)
        -
        \min_{ x:\|x\|_2\leq K} \sum\limits_{t=1}^{T}f_t( x)
        \leq
        3KG\sqrt{T}.
    $$
\end{lemma}
\begin{proof}
We refer to Theorem 3.1 in \cite{hazan2016introduction}.
\end{proof}

Next, we show Proposition \ref{prop:sgd}.
\begin{proof}
    Recall the definition of $l_t(D_t,\Theta)$ as follows
    $$
        l(D_{t};\Theta) =  \left\|(1_{|\mathcal{N}^*_t|} - \hat{c}_{t,\mathcal{N}^*_t}^\top - \hat{c}_{t,\mathcal{B}^*_t}^\top A_{t,\mathcal{B}^*_t}^{-1} A_{t,\mathcal{N}^*_t})_+\right\|_1,
    $$ 
    where $\hat{c}=\Theta z_t$, and $(\cdot)_+$ denotes the entry-wise positive part function. For the sake of simplicity, we use $l_t(\Theta)$ to denote $l(D_{t};\Theta)$. By calculating the derivative of $l_t(\Theta)$, we have
    \begin{align}
    \label{eq:devl}
        \nabla_{{\Theta}_{\mathcal{N}_t}} l_t(\Theta) =
    -{g}_t{z}_{t}^{\top},\ 
    \nabla_{{\Theta}_{\mathcal{B}_t}} l_t(\Theta) =
     ({A}_{t,\mathcal{B}_t}^{-1})^{\top}{A}_{t,\mathcal{N}_t}{g}_{t}{z}^{\top}.
    \end{align}
    Here, 
    ${g}_t\in\mathbb{R}^{n-m}$ is defined as follows. For the $i$-th element in $\mathcal{N}_t$ for $i=1,...,m$, correspondingly, we define the $i$-th element of ${g}_t$ as
    \begin{align*}
        {g}_{t,i} = 
        \left\{
        \begin{matrix}
            1, &\text{ if $\left(\left(
         {e}_{\mathcal{N}_t}
        -
         {\Theta}_{\mathcal{N}_t} {z}_t
        +
         {A}_{t,\mathcal{N}_t}^{\top}( {A}_{t,\mathcal{B}_t}^{-1})^{\top} {\Theta}_{\mathcal{B}_t} {z}_t
    \right)_+\right)_{i}>0$}\\
            0, &\text{ otherwise.}
        \end{matrix}
        \right.
    \end{align*}
    Then, we have 
    \begin{align*}
        \|\nabla l_t(\Theta)\|_F
        &\leq
        \|g_tz_t^\top\|_F+\|({A}_{t,\mathcal{B}_t}^{-1})^{\top}{A}_{t,\mathcal{N}_t}{g}_{t}{z}^{\top}\|_F\\
        &\leq
            \|g_t\|_2\|z_t\|_2+\|{A}_{t,\mathcal{B}_t}^{-1}\|_F\|{A}_{t,\mathcal{N}_t}\|_F\|g_t\|_2\|z_t\|_2\\
        &\leq
        \sqrt{n}+\|{A}_{t,\mathcal{B}_t}^{-1}\|_F\|{A}_{t,\mathcal{N}_t}\|_F\cdot\sqrt{n}\\
        &\leq
        \sqrt{n}+\bar{\sigma}mn,
    \end{align*}
    where the first line comes from the equalities in \eqref{eq:devl} and the triangle inequality inequality $\|A+B\|_F\leq\|A\|_F+\|B\|_F$ for any two matrices $A,B$ with the same size, the second line comes from the inequality $\|AB\|_F\leq\|A\|_F\|B\|_F$ for any two matrices $A,B$ and the fact that $\|g\|_F=\|g\|_2$ for any vector $g$, the third line comes from the definition of $g_t$ and Assumption \ref{assp:bdd} that $\|z_t\|_2\leq 1$, and the last line comes from Assumption \ref{assp:bdd} that all entries of $A_t$ are in $[-1,1]$ and the inequality $\|A\|_F\leq\sqrt{m}\sigma_{\max}(A)$ for any matrix $A\in\mathbb{R}^{m\times m}$.
    Next, by Lemma \ref{lem:sgdconv}, with $\eta=\frac{2\bar{\Theta}}{(\sqrt{n}+\bar{\sigma}mn)\sqrt{T}}$,
    \begin{align}
        \label{ieq:sgdsvm}
        \sum\limits_{t=1}^{T}l_t(\Theta_t)
        -
        \min_{ \Theta:\|\Theta\|_F\leq \hat{\Theta}} \sum\limits_{t=1}^{T}l_t( \Theta)
        \leq
        \left(3\bar{\Theta}\sqrt{n}+3\bar{\sigma}\bar{\Theta}mn\right)\sqrt{T}.
    \end{align}
    Then, we show inequality \eqref{bound2} by
    \begin{align*}
        \frac{1}{T}\mathbb{E}[z_t^\top\Theta_t^{\top}(x^*_t-x_t)]
        &\leq
        \mathbb{E}\left[\sum\limits_{t=1}^{T}l_t(\Theta_t)\right]\\
        &\leq
        \mathbb{E}\left[\min_{ \Theta:\|\Theta\|_F\leq \hat{\Theta}} \sum\limits_{t=1}^{T}l_t( \Theta)\right]+\frac{3\bar{\Theta}\sqrt{n}+3\bar{\sigma}\bar{\Theta}\cdot{mn}}{\sqrt{T}}.
    \end{align*}
    Here, we can obtain the first line by Lemma \ref{lem:optc} and a similar proof as in the proof of Proposition \ref{prop:ofsprt}, and obtain the second line by inequality \eqref{ieq:sgdsvm}. 
    
    Furthermore, under Assumption \ref{assp:theta}, we have with probability 1,
    \begin{align*}
        \min_{ \Theta:\|\Theta\|_F\leq \hat{\Theta}} \sum\limits_{t=1}^{T}l_t( \Theta)=0,
    \end{align*}
    which implies
    \begin{align}
        \label{eq:sgdsep}
        \mathbb{E}\left[\sum\limits_{t=1}^{T}l_t(\Theta_t)\right]
        &\leq
        \frac{3\bar{\Theta}\sqrt{n}+3\bar{\sigma}\bar{\Theta}\cdot{mn}}{\sqrt{T}}.  
    \end{align}
    Recall that $\hat{\Theta}$ denotes the matrix sampled uniformly from $\{\Theta_t\}_{t=1}^{T}$, ${x}^*_{\text{new}}$ denotes the optimal solution of $\text{LP}({c}_{\text{new}},A_{\text{new}},b_{\text{new}})$, $\hat{x}_{\text{new}}$ denotes the optimal solution of $\text{LP}(\hat{c}_{\text{new}},A_{\text{new}},b_{\text{new}})$, and $\hat{c}_{\text{new}}=\hat{\Theta}z_{\text{new}}$.  Similar to the last paragraph of the proof of Proposition \ref{prop:ofsprt}, we have if $\hat{\Theta}$ misclassifies any non-basic variable in $\mathcal{N}_{\text{new}}$, we have $l(D_{\text{new}};\hat{\Theta})\geq1$. Thus, 
    \begin{align*}
        \mathbb{E}[\mathbb{P}(x^{*}_{new}\not=\hat{x}_{new})]
        &\leq
        \mathbb{E}[\mathbb{P}(l(D_{\text{new}};\hat{\Theta})\geq1)]\\
        &\leq
        \mathbb{E}[l(D_{\text{new}};\hat{\Theta})]\\
        &=
        \frac{1}{T+1}\mathbb{E}\left[\sum\limits_{t=1}^{T+1}l(D_{\text{new}};{\Theta}_t)\right]\\
        &=
        \frac{1}{T+1}\mathbb{E}\left[\sum\limits_{t=1}^{T+1}l(D_t;{\Theta}_t)\right]\\
        &\leq
        \frac{3\bar{\Theta}\sqrt{n}+3\bar{\sigma}\bar{\Theta}\cdot{mn}}{\sqrt{T+1}},
    \end{align*}
    by which we have shown that, with probability no less than $1-\frac{3\bar{\Theta}\sqrt{n}+3\bar{\sigma}\bar{\Theta}\cdot{mn}}{\sqrt{T+1}}$, Algorithm \ref{alg:sgd} can identify both the true optimal basis and the true optimal solution correctly. Here, the first line comes from the fact that $l(D_{\text{new}};\hat{\Theta})\geq1$ if $x^{*}_{new}\not=\hat{x}_{new}$, the second line comes from Markov's inequality, the third line comes from the definition of $\hat{\Theta}$, the forth line comes from the fact that $D_t$ and $D_{new}$ are two i.i.d. samples that are also  independent of $\Theta_t$, and the last line comes from inequality \eqref{eq:sgdsep}.
    
\end{proof}

\subsection{Proof of Proposition \ref{prop:pcpt}}
The proof follows the standard analysis of the perceptron method.

\begin{proof}
    In this part, we denote $\Theta_{t,i}$ as the value of matrix $\Theta_{\text{tmp}}$ at the beginning on the $t$-th iteration of the outer loop and the $i$-th iteration of the inner loop for $t\in[T]$ and $i\in[n]$. Also, we view $\Theta_{t,n+1}$ and $\Theta_{t+1,1}$ as the same to avoid undefined boundary cases. Recall the definition of the reduced cost vector
    $$
        r_t(\Theta)=\Theta z_t-A_t^{\top}(A_{t,{\mathcal{B}}_t}^{-1})^{\top}\Theta_{{\mathcal{B}}_t} z_t, \text{ for $t=1,...,T$},
    $$
    which is a linear function of $\Theta$ entry-wisely. Thus, we have that for each $t=1,...,T$ and $i\in[n]$, there exists a matrix $W_{t,i}\in\mathbb{R}^{n*d}$ such that 
    $r_{t,i}^{(i)}(\Theta)=\text{Trace}(W_{t,i}^{\top}\Theta)$ and $\|W_{t,i}\|_F\leq 1+\bar{\sigma}\sqrt{mn}$, where Trace$(W)=\sum\limits_{i=1}^{n} w_{ii}$ for any square matrix $W=(w_{ij})_{i,j=1}^{n}\in\mathbb{R}^{n\times n}$. Then, we define 
    $$
        h_{t,i}(\Theta) = \text{sign}(\text{Trace}(W_{t,i}^{\top}\Theta)-.5), 
    $$
    where $\text{sign}(\cdot)$ denotes the sign function. Moreover, if there is an $i\in\mathcal{N}_t$ at some time $t$ such that $h_{t,i}(\Theta_{t,i})=-1$, we have
    $\Theta_{t,i+1}=\Theta_{t,i}+W_{t,i}$. The updating rule in Algorithm \ref{alg:prcpt} is obtained as mentioned above. Specifically, as in Algorithm \ref{alg:prcpt}, for any $i\in\mathcal{N}_t$ and $t\in[T]$, $W_{t,i}$ is defined as follows $$(W_{t,i})_i=z_t^{\top},\ (W_{t,i})_{\mathcal{B}_t^*}=A_{\mathcal{B}_t^*}^{-1}A_tz_t,$$
    and all other entries are 0. Then, we have
    \begin{align}
        \label{ieq:wbd}
        \|W_{t,i}\|_F^2
        &=\|z_t\|_F^2+\|A_{\mathcal{B}_t^*}^{-1}A_tz_t\|_F^2 \nonumber\\
        &\leq
        \|z_t\|_F^2+\|A_{\mathcal{B}_t^*}\|_F^2\|A_t\|_F^2\|z_t\|_F^2,\\
        &\leq
        1+\bar{\sigma}^2m^2n\nonumber
    \end{align}
    where the first line comes directly from the definition of $W_{t,i}$, the second line comes from the inequality that $\|AB\|_F\leq\|A\|_F\|B\|_F$ for any two matrices $A,B$, and the last line comes from Assumption \ref{assp:bdd}.
    
    We say that Algorithm \ref{alg:prcpt} misclassifies one non-basic variable $i\in\mathcal{N}_t$ if $h_{t,i}(\Theta_{t,i})\leq 0$, and misclassifies one basic variable $i\in\mathcal{B}_t$ if $h_{t,i}(\Theta_T)> 0$. Since the values of entries of the reduced cost vector corresponding to the true basis $\mathcal{B}_t$ are $0$ for all $t$ and $i$, algorithm \ref{alg:sgd} makes a mistake only if it misclassifies one non-basic variable. 
    Denote the number of identification mistakes at the $t$-th iteration as $K_t$ for $t=1,...,T$. Let $K=\sum\limits_{t=1}^{T}K_t$ be the number of all mistakes. From the updating rule, inequality \eqref{ieq:wbd} and the triangle inequality of the Frobenius norm, we have
    $
        \|\Theta_{t,i+1}\|_F^2
        \leq
        K_t(1+\bar{\sigma}^2m^2n),
    $
    and then,
    \begin{align}
    \label{ieq:pcpthigh}
        \|\Theta_{T+1}\|_{F}^2
        \leq
        K(1+\bar{\sigma}^2m^2n).
    \end{align}

    Moreover, under Assumption \ref{assp:theta}, there exists a matrix $\Theta^*$ such that 
    \begin{align*}
            \begin{matrix}
                \text{Trace}(W_{t,i}^{\top}\Theta^*)\geq 1, &\text{ if $i\in\mathcal{N}_t$},\\
                \text{Trace}(W_{t,i}^{\top}\Theta^*)\leq 0, & \text{ if $i\in\mathcal{B}_t$},
            \end{matrix}
    \end{align*}
    for all $t=1,...,T$. Then, we have once a mistake is made for some $i\in\mathcal{N}_t$ at time $t$
    \begin{align*}
        \text{Trace}((\Theta_{t,i+1}-\Theta_{t,i})^{\top}\Theta^*)
        =
        \text{Trace}(W_{t,i}^{\top}\Theta^*)
        \geq
        1,
    \end{align*}
    which implies
    \begin{align}
    \label{ieq:pcptlow}
        \text{trace}(\Theta_{T,n+1}^{\top}\Theta^*)
        \leq
        K.
    \end{align}
    Then, combining inequalities \eqref{ieq:pcpthigh} and \eqref{ieq:pcptlow}, we have
    \begin{align*}
        K &\leq \text{trace}(\Theta_{T,n+1}^{\top}\Theta^*)\\
        &\leq
        \bar{\Theta}\|\Theta_{T,n+1}\|_F\\
        &\leq
        \bar{\Theta}\sqrt{K(1+\bar{\sigma}^2m^2n)},
    \end{align*}
    where the first line comes from \eqref{ieq:pcptlow}, the second line comes from Assumption \ref{assp:theta} that $\|{\Theta}^*\|_F\leq\bar{\Theta}$ and Cauchy inequality, and the last line comes from \eqref{ieq:pcpthigh}. Dividing each side by $\sqrt{K}$ and taking square,
    $$
        K\leq
        \bar{\Theta}^2+\bar{\sigma}^2\bar{\Theta}^2m^2n.
    $$
    
    Moreover, Lemma \ref{lem:opt_con} tells that one can recover the optimal solution if the optimal basis is identified. This statement implies that $K$ is an upper bound of times that we cannot identify the true optimal solutions by Algorithm \ref{alg:prcpt}. Thus, we have
    $$
        |\{t\in[T]:x_t^*\not=x_t\}|\leq \bar{\Theta}^2+\bar{\sigma}^2\bar{\Theta}^2m^2n.
    $$
    
    For the generalization bound, we apply the symmetry of samples, follow similar steps in the last paragraph of the proof of Proposition \ref{prop:sgd} and have
    $$
        \mathbb{E}(\mathbb{P}(x^*_{\text{new}}\not=\hat{x}_{\text{new}}))
        \leq
        \frac{\bar{\Theta}^2+\bar{\sigma}^2\bar{\Theta}^2m^2n}{T}.
    $$

\end{proof}

\section{Additional Discussions}

\subsection{Why structured prediction does not work}

\label{sec_struct_svm}

One might wonder if our Maximum Optimality Margin approach can be directly solved as a structured classification problem by using structured SVM classifier. In this section, we will argue that the classical ways of structured SVM without any surrogate loss functions are computationally intractable.

The goal of such structured SVM classifier to estimate the optimal basis (or equivalently, the non-basic variables) by observing $\{(z,A)\}$ and a predictor trained on $\{(z_t, A_t)\}$'s and their corresponding labels $\{y_t\}$'s, where we define $(y_t)_i = +1$ for $i \in \mathcal{N}_t$ and $(y_t)_i = -1$ for $i \in \mathcal{B}_t$. The Maximum Optimality Margin approach is a principle to maximize the estimated reduced cost vectors for non-basic variables. To express this principle more explicitly, the margin term one wants to maximize in structured SVM can be written as
$$ \max_{\Theta \in \mathcal{K}} \sum_{j \in \mathcal{N}}\hat{r}_{j}, $$
where $\hat{r}_j$'s are to be specified later.

For the simplicity of notations, we define some auxiliary vectors and matrices named as $\bm{1}$, $J_{\mathcal{B}_t}$, $J_{\mathcal{N}_t}$, and $\Phi_{t,y_t}$, where
$$ \bm{1}\in \mathbb{R}^{n-m}, \ \bm{1}_j = 1, \forall j \in [n-m], $$
$$ J_{\mathcal{B}_t}v = v_{\mathcal{B}_t}, \ \forall v \in \mathbb{R}^n, $$
$$ J_{\mathcal{N}_t}v = v_{\mathcal{N}_t}, \ \forall v \in \mathbb{R}^n, $$
$$ \Phi_{t,y_t} \coloneqq J_{\mathcal{N}_t} - A_{t,\mathcal{N}_t}^\top A_{t,\mathcal{B}_t}^{-\top} J_{\mathcal{B}_t}. $$
Specifically, the estimated reduced cost with respect to $y_t$ can be written as
$$ \hat{r}_{t,y_t} = \Phi_{t,y_t} \Theta z_t. $$
Hence the margin term (where the subscript $t$ is sometimes omitted for simplicity) is
$$ \sum_{j \in \mathcal{N}}\hat{r}_{j} = \bm{1}^\top \hat{r} = \langle \Theta, \Phi_{y}^\top \bm{1} z^\top \rangle, $$
which implies the feature map
$$ \phi((z_t, A_t), y_t) = \Phi_{t,y_t}^\top \bm{1} z_t^\top. $$
Note that the corresponding label space 
$$\mathcal{Y} = \left\{ y \in \{-1, +1\}^n, \#\{t, y_t = +1\} = m\right\}$$
is exponentially large in $n$ in general cases. We also define some measurement of difference $\Delta(\cdot, \cdot) \in \mathcal{Y} \times \mathcal{Y} \rightarrow \mathbb{R},$ where $\Delta(y, y^\prime) \geq 0$ and $\Delta(y,y) = 0$. For example, the $\Delta(y,y^\prime)$ function can be $\mathbbm{1} \{y \neq y^\prime\}$ and we retrieve the multiclass Hinge loss. Such $\Delta(y,y^\prime)$ can also be defined to be the Hamming distance and so on.

Equipped with those notations,
the structured SVM problem can now be formulated as:
\begin{align*}
    \min_{\Theta\in \mathcal{K}, s\in\mathbb{R}^T, y\in \mathcal{Y}}\ & \frac{\lambda}{2} \|\Theta\|^2 + \frac{1}{T} \sum_{t=1}^T s_t,\\
    \text{s.t. } & s_t \geq \Delta(y_t, y) - \langle \Theta, \phi((z_t, A_t), y_t)\rangle + \langle \Theta, \phi((z_t, A_t), y)\rangle, \quad \forall t \in [T],\ \forall y \in \mathcal{Y},\\
    & s_t \geq 0,\quad \forall t \in [T].
\end{align*}
We thereby note that solving the above problem requires solving another sub-problem where
$$ g_t(\Theta) \coloneqq s_t^* = \max_{y \in \mathcal{Y}} \left\{ \Delta(y_t, y) - \langle \Theta, \phi((z_t, A_t), y_t)\rangle + \langle \Theta, \phi((z_t, A_t), y)\rangle \right\}.  $$

But the above sub-problem is computationally intractable since the precise evaluation of $g_t(\Theta)$ requires solving a discrete optimization problem with exponentially large feasible set $\mathcal{Y}$. Such an obstacle makes the training hard to implement. Besides, even if we get the training result $\tilde{\Theta}$, the following inference problem
$$ \tilde{f}(z,A) = \argmax_{y\in \mathcal{Y}} \langle \tilde{\Theta}, \phi((z,A),y)\rangle $$
is still highly intractable.

\subsection{Scale consistency of least squares linear regression}
\label{sec:consist_of_lr}

In the previous sections, we raise the point of scale consistency several times. Specifically, the key point made is that the objective vectors of $c$ and $\alpha c$ for $\alpha>0$ produce the same optimal solution. Therefore, the algorithm should account for this scale invariance as there might be some scale contamination of the training data in application contexts such as revealed preference/stated preference. Here we present a self-contained result on the scale consistency of the linear regression model that might be of independent interests.

Consider the linear regression model where one wants to estimate the true coefficient matrix using cost vectors $c_t \in \mathbb{R}^n$ and feature vectors $z_t \in \mathbb{R}^d$. Instead of observing true $c_t$'s, we only observe their perturbed/contaminated versions with scale noises $(1+\alpha_t) c_t$'s, where $\alpha_t$'s here are some random variables. We claim that if $\alpha_t$'s are i.i.d. generated and independent of $z_t$ and $c_t$, then the ordinary least squares method is still consistent if $\mathbb{E}[\alpha] = 0$.
\begin{proposition}[Scale Consistency of Ordinary Least Squares]
Assume $c = {\Theta}^* z + \epsilon$, where $\mathbb{E}[\epsilon | z] = 0$, ${\Theta}^* \in \mathbb{R}^{n \times d}$ is the true underlying coefficient matrix. Assume $\alpha$ is independent of $z$ and $c$. Assume we observe $T$ i.i.d. samples $( z_t, (1+\alpha_t) c_t)$'s. Assume $z_t$, $c_t$, and $\alpha_t$ have finite second-order moments, which implies that they follow the strong law of large numbers. Further assume that $c_t$ and $\alpha_t$ have finite fourth order moment, which implies the strong law of large numbers for $c_t^2$ and $\alpha_t^2$. Assume $\mathbb{E}[z_t z_t^\top] = \Sigma$, and $\Sigma$ is non-singular. We estimate the underlying coefficient matrix by the ordinary least squares method:
$$ \hat{\Theta}_T = \argmin_{\Theta} \left\{f_T(\Theta) \coloneqq \sum_{t=1}^T \frac{1}{T} \|(1+\alpha_t)c_t - \Theta z_t\|_2^2\right\}.$$

Then 
$$\hat{\Theta}_T \xrightarrow[]{\text{a.s.}} (1+\mathbb{E}[\alpha]){\Theta}^*, \quad \text{as }T \rightarrow \infty. $$
\end{proposition}
\begin{proof}
W.l.o.g. we only prove the one-dimensional case $c \in \mathbb{R}$, since multi-dimensional cases can be proved similarly by breaking $\Theta = (\Theta_1, \dots, \Theta_n)^\top$ into $n$ independent $\Theta_i$'s. For the one-dimensional case, the estimator is now minimizing
\begin{align*}
    f_T(\Theta) & = \frac{1}{T} \sum_{i=1}^T \|z_t^\top \Theta - (1+\alpha_t)c_t\|_2^2\\
    & = \Theta^\top \left[\frac{1}{T} \sum_{i=1}^T z_t z_t^\top\right] \Theta - \left[\frac{1}{T}\sum_{t=1}^T 2(1+\alpha_t) c_t z_t^\top\right]\Theta + \frac{1}{T}\sum_{t=1}^T (1+\alpha_t)^2 c_t^2.
\end{align*}
Note that we assume that $z_t$'s, $c_t$'s, $c_t^2$'s, $\alpha_t$'s, $\alpha_t^2$'s all have finite second-order moments, which implies
$$ \frac{1}{T} \sum_{t=1}^T z_t z_t^\top \xrightarrow[]{\text{a.s.}} \Sigma,\quad \text{as } T\rightarrow \infty, $$
$$ \frac{1}{T} \sum_{t=1}^T 2(1+\alpha_t)c_t z_t^\top \xrightarrow[]{\text{a.s.}} 2\mathbb{E}[(1+\alpha)c z^\top], \quad \text{as } T\rightarrow \infty, $$
$$ \frac{1}{T}\sum_{t=1}^T (1+\alpha_t)^2 c_t^2 \xrightarrow[]{\text{a.s.}} \mathbb{E}[(1+\alpha)^2 c^2] = (1+2\mathbb{E}[\alpha] + \mathbb{E}[\alpha^2])\mathbb{E}[c^2] < \infty, \quad \text{as } T\rightarrow \infty. $$
Combining the boundedness of the last component with the fact that $\Sigma$ is non-singular and positive semidefinite and the fact that $2\mathbb{E}[(1+\alpha)c z^\top]$ is bounded (which will be shown later), we have 
$$f_T \xrightarrow[]{\text{a.s.}} f, \quad \text{as }T\rightarrow \infty,$$
where $f$ is a positive definite quadratic function of $\Theta$.

Therefore, the unique minimizer of $f$ must be its first-order stationary point. We now compute the partial derivatives. We have
\begin{align*}
    \mathbb{E}[(1+\alpha)c z^\top] & = \mathbb{E}[1+\alpha]\mathbb{E}[cz^\top]\\
    & = (1+\mathbb{E}[\alpha]) \mathbb{E}[\mathbb{E}[cz^\top|z]]\\
    & = (1+\mathbb{E}[\alpha]) \mathbb{E}[\mathbb{E}[{\Theta}^{*\top} z z^\top + \epsilon z^\top|z]]\\
    & = (1+\mathbb{E}[\alpha]) \mathbb{E}[{\Theta}^{*\top} z z^\top],\\
    & = (1+\mathbb{E}[\alpha]) {\Theta}^{*\top} \Sigma.
\end{align*}
where the first equality is from the fact that $\alpha_t$'s are independent of $z_t$'s and $c_t$'s, the second equality follows the tower law of conditional expectation, the third equality comes from the linear assumption, and the fourth equality comes from $\mathbb{E}[\epsilon|z] = 0$. It follows immediately that
$$ \frac{\partial f}{\partial \Theta} = 2 \Sigma \Theta - 2(1+\mathbb{E}[\alpha])\Sigma {\Theta}^*. $$
Since $\Sigma$ is non-singular, we have proved that the unique minimizer of $f$ is exactly $(1+\mathbb{E}[\alpha]) {\Theta}*$. 

Then from the fact that $f_T \xrightarrow[]{\text{a.s.}} f$ and $f$ is positive definite quadratic function, we have
$$ \hat{\Theta}_T = \argmin f_T \xrightarrow[]{\text{a.s.}} \argmin f = (1+\mathbb{E}[\alpha]){\Theta}^*, \quad \text{as }T \rightarrow \infty. $$
Specifically, if $\mathbb{E}[\alpha] = 0$, we retrieve the true ${\Theta}^*$.
\end{proof}

\end{document}